\documentclass[10pt, letterpaper]{article}
\usepackage{arxiv}

\title{
    Provably Mitigating Overoptimization in RLHF: Your SFT Loss is Implicitly an Adversarial Regularizer
}
\author{
    Zhihan Liu\thanks{Equal contributions.} \thanks{Northwestern University.\texttt{\{zhihanliu2027,shenaozhang2028,hongyiguo2025\}@u.northwestern.edu,zhaoranwang@gmail.com}} \qquad
    Miao Lu\footnotemark[1] \thanks{Stanford University. \texttt{\{miaolu, jose.blanchet\}@stanford.edu}} \qquad
    Shenao Zhang\footnotemark[2]\qquad 
    Boyi Liu\thanks{ByteDance Inc. \texttt{\{boyi.liu01, yingxiang.yang\}@bytedance.com}} \\ 
    Hongyi Guo\footnotemark[2]\qquad
    Yingxiang Yang\footnotemark[4]\qquad
    Jose Blanchet\footnotemark[3]\qquad 
    Zhaoran Wang\footnotemark[2]
}
\date{\today}

\date{\small{\today}}

\begin{document}


\maketitle

\begin{abstract}
   Aligning generative models with human preference via RLHF typically suffers from overoptimization, where an imperfectly learned reward model can misguide the generative model to output undesired responses. 
   We investigate this problem in a principled manner by identifying the source of the misalignment as a form of distributional shift and uncertainty in learning human preferences. 
   To mitigate overoptimization, we first propose a theoretical algorithm that chooses the best policy for an adversarially chosen reward model; one that simultaneously minimizes the maximum likelihood estimation of the loss and a reward penalty term. 
   Here, the reward penalty term is introduced to prevent the policy from choosing actions with spurious high proxy rewards, resulting in provable sample efficiency of the algorithm under a \emph{partial coverage} style condition.
  Moving from theory to practice, the proposed algorithm further enjoys an equivalent but surprisingly easy-to-implement reformulation. Using the equivalence between reward models and the corresponding optimal policy, the algorithm features a simple objective that combines: (i) a preference optimization loss that directly aligns the policy with human preference, and (ii) a supervised learning loss that explicitly imitates the policy with a (suitable) baseline distribution. 
  In the context of aligning large language models (LLM), this objective fuses the direct preference optimization (DPO) loss with the supervised fine-tuning (SFT) loss to help mitigate the overoptimization towards undesired responses, for which we name the algorithm \underline{R}egularized \underline{P}reference \underline{O}ptimization (RPO).
  Experiments of aligning LLMs demonstrate the improved performance of RPO compared with DPO baselines. 
  Our work sheds light on the interplay between preference optimization and SFT in tuning LLMs with both theoretical guarantees and empirical evidence.
\end{abstract}


\tableofcontents

\section{Introduction}

A key step in building state-of-the-art LLMs is Reinforcement Learning from Human Feedback (RLHF) \citep{christiano2017deep, ziegler2019fine}, which aligns pretrained LLMs with human preferences using human assessment data, making the model more helpful, truthful, and harmless \citep{ouyang2022training, casper2023open}.
Without doing RLHF, pretrained LLMs could exhibit harmful behaviors including offensive or toxic outputs, social biases, and leaking sensitive information from training data, etc \citep{ gehman2020realtoxicityprompts, carlini2021extracting, ganguli2022red}.
Typically, RLHF first learns a reward model from data (pair-wise comparisons of responses) to quantify the human preferences of LLM outputs.
Then it fine-tunes the LLM to maximize the learned reward using RL techniques.

In this pipeline, a crucial challenge is \emph{reward overoptimization} or \emph{reward hacking} \citep{michaud2020understanding, tien2022causal, gao2023scaling}.
Since the reward model is learned from finite data, this reward might not be perfectly aligned with the underlying human preference. 
Optimizing the LLM towards such an imperfectly learned and potentially overfitted reward model leads to performance degeneration and a substantial decrease in the probability of choosing the preferred responses in the data \citep{hong2024orpo, rafailov2024r}.
Given the importance of RLHF and the outlined challenge, a crucial research question is:
\begin{center}
\emph{
How to mitigate reward overoptimization in RLHF in a principled \\
and efficient manner for better alignment?}
\end{center}

To answer the question, we model RLHF as an offline contextual bandit \citep{ouyang2022training} and ascribe overoptimzation to distributional shifts and reward uncertainty.
Intuitively, when fine-tuning an LLM, the response (action) distribution of the tuned LLM could deviate from that of the training data.
For the out-of-distribution responses, which are dissimilar with (or not well covered by) the responses in the data, the high inherent uncertainty of underlying human preferences could make the learned reward model misleading for out-of-distribution responses. 
In this situation, reward overoptimization can occur because the LLM is fine-tuned towards maximizing a reward model with defective out-of-distribution prediction, giving a potential consequence that the LLM responses are favored by the learned reward but
less preferred by a human \citep{zhu2024iterative}. 
We illustrate such an issue inherent to overoptimization in Figure~\ref{fig: 1}.

In this work, we propose a new RLHF algorithm to mitigate reward overoptimization.
From a high level, our theoretical algorithm seeks the best LLM for an \emph{adversarially} chosen reward model that minimizes the sum of its maximum likelihood estimation (MLE) loss and its own expected reward value.
Intuitively, since the reward value is also minimized when minimizing the sum, it can automatically prevent the misleadingly high reward caused by the uncertainty inherent in learning from finite preference data.
Furthermore, we show that the theoretical algorithm enjoys an easy implementation: it simply adopts a supervised fine-tuning (SFT) loss as a regularizer during training. By explicitly regularizing the LLM to imitate high-quality responses (e.g., the preferred responses in dataset), the algorithm can effectively mitigate the issue of overoptimization. We provide both theory and experiments to demonstrate our findings, which we summarize next.

\begin{figure}[!t]
    \begin{minipage}{0.31\linewidth}
            \centering
            \hspace{-1mm}\includegraphics[height=0.80\textwidth]{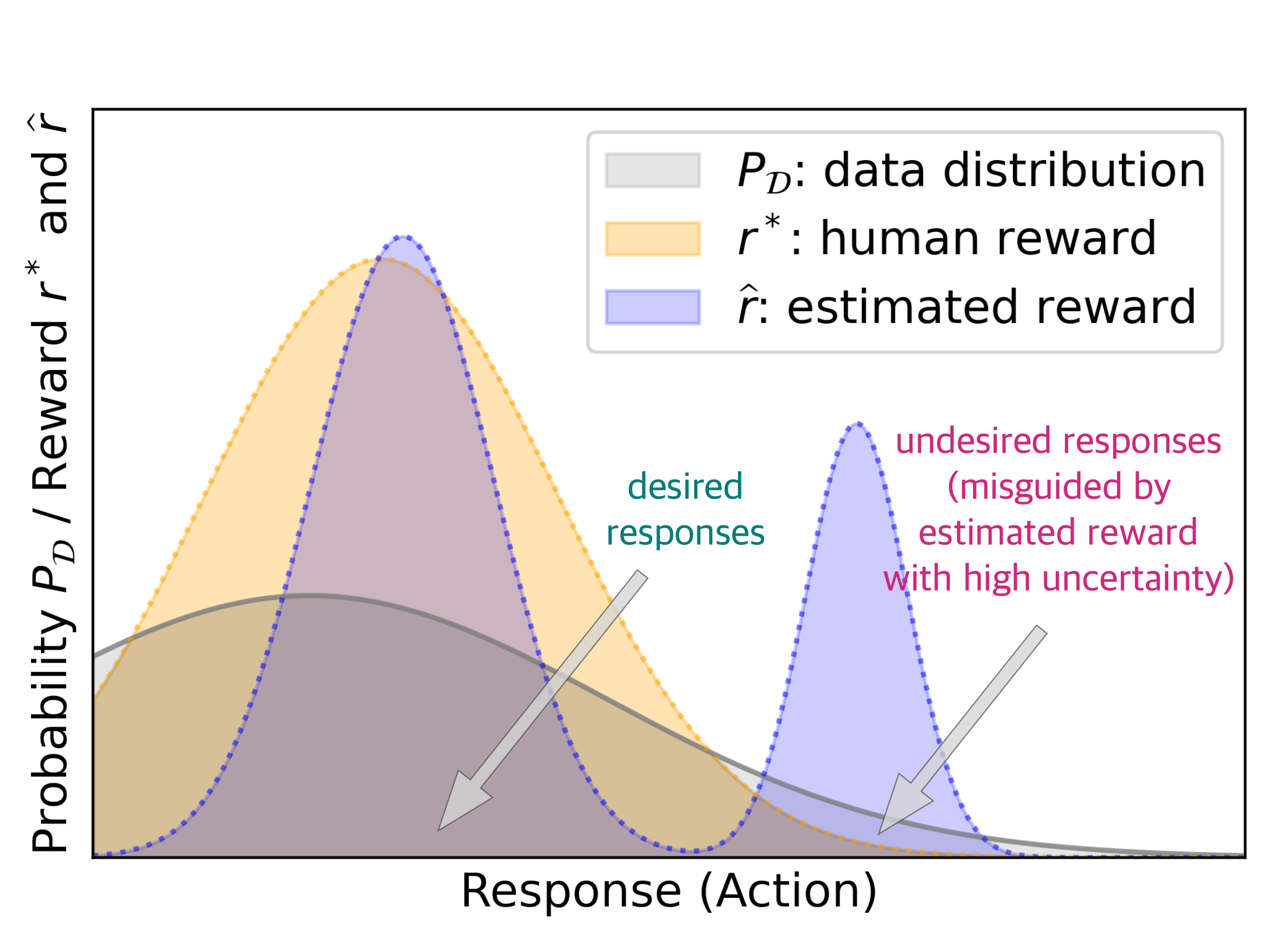}
    \end{minipage}
    \hspace{5mm}
    \begin{minipage}{0.31\linewidth}
            \centering
            \vspace{2.5mm}
        \includegraphics[height=0.71\textwidth]{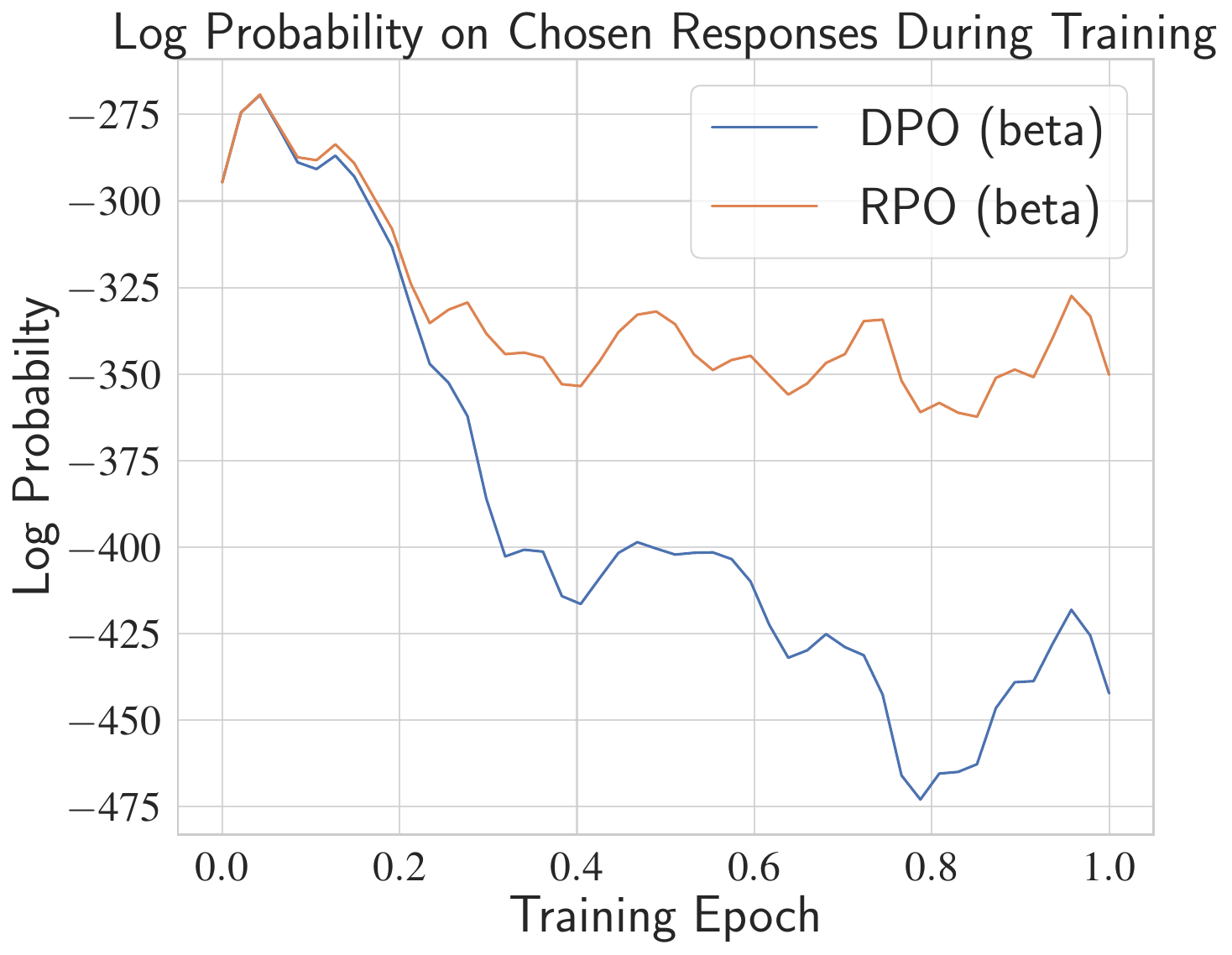}
    \end{minipage} 
    \hspace{2mm}
    \begin{subfigure}
        \centering
        {
            \begin{tabular}{c|ccc}
                \toprule
                    action & $a$ & $b$ & $c$ \\
                    \midrule 
                    \midrule
                    $r^{\star}$&  1 & 0.5 & 0\\
                    Dataset  & \multicolumn{3}{c}{$\mathcal{D} = \{(a,b,1)\}$}\\
                    \midrule
                    $\pi^{\mathrm{ref}}$ &  $0.45$ & $0.45$ &$0.1$\\
                     $\pi^{\mathrm{DPO}}$& $0.5$ & $0$ & \textcolor{red}{\bf 0.5}  \\
                      $\pi^{\mathrm{RPO}}$& $\textcolor{blue}{\bf 1.0}$& $0$ & $0$  \\
                \bottomrule
            \end{tabular}%
        }
    \end{subfigure}
    \caption{\small \textbf{Left:} Reward overoptimization due to distributional shift and uncertainty in reward. \textbf{Middle:} Overoptimization can cause the probability of outputting the preferred responses from preference data to decrease substantially using original DPO proposed by \citep{rafailov2024direct}. 
    Our algorithm (RPO) can significantly alleviate this decrease. 
    See more discussions in Section~\ref{sec: experiment}. 
    \textbf{Right:} For the illustration of the intuition behind the success of RPO, we give a simple example with no prompts and only three responses $a$, $b$, and $c$ \citep{xu2024dpo}. 
    The dataset $\cD$ only consists of a single data point where response $a$ is preferred than response $b$ (as indicated by the label $1$).
    Writing out the DPO loss, one needs to minimize $\log(1+(\pi(b)/\pi(a))^\beta)$ for $0\leq \pi(a)+\pi(b)\leq 1$, which gives $\pi(b)=0$ and leaves $\pi(a)$ arbitrary. 
    Thus possible solutions to DPO objective would bring out non-zero probabilities of response $c$ that is not desirable since $r^{\star}(c)=0$.
    In contrast RPO objective simultaneously maximizes another term $\log\pi(a)$ as regularizer, avoiding the output of response $c$.
    }\label{fig: 1}
\end{figure}

\subsection{Our Contributions}

We summarize our contributions in three areas as follows.

\paragraph{A theoretical algorithm under general function approximation.}
Our first contribution is a new theoretical algorithm (Algorithm~\ref{alg0}).
It features an \emph{unconstrained maximin} problem, outputting the optimal policy (LLM) against an adversarially chosen reward model that minimizes the summation of: 
\texttt{\textcolor{blue}{(a)}} the MLE loss for estimating the underlying reward; and
\texttt{\textcolor{blue}{(b)}}  a reward expected value term as a penalty that aims to prevent spuriously high reward estimation caused by data uncertainty and insufficient coverage.
Algorithm~\ref{alg0} is compatible with general function approximations of the reward model, meaning that we do not impose any specific structural
form to the hypothesis class of reward, demonstrating its generality.

In this regime, we establish the finite-sample suboptimality gap of Algorithm~\ref{alg0} as $\widetilde{\cO}(C^2_{\mathrm{coverage
}}\sqrt{\cN_{\cR}/N})$ when competing with any LLM in terms of the underlying human reward (Theorem~\ref{thm:main}).
Here $N$ denotes the number of human comparison data. 
$\cN_{\cR}$ denotes the complexity of the reward model class $\cR$, and $C_{\mathrm{coverage}}$ characterizes the coverage of the preference dataset with respect to the response distribution of the LLM to compete (please see Assumption~\ref{as: coverage} for details).
This indicates that, as long as the training data well cover the LLM $\pi$ to compete, the algorithm is guaranteed to align an LLM to output responses as good as $\pi$ in terms of human reward, without suffering from overoptimization caused by distributional shifts and inherent uncertainty in human preference.

\paragraph{An easy-to-implement practical objective.}
Moving towards practice, we show that the objective of Algorithm~\ref{alg0} adopts a surprisingly simple and equivalent form for its use in practice.
Specifically, with mild regularity conditions, we prove that the \emph{maximin} objective (Algorithm\ref{alg0}) is equivalent to the corresponding \emph{minimax} objective, which is further reduced to a single minimization problem for the reward model since its inner problem adopts a closed form solution.
Inspired by recent RLHF research that uses reward-model-free methods to align LLMs \citep{rafailov2024direct}, 
we further re-parameterize the reward model via its corresponding KL-regularized optimal policy.
Then the minimization objective of the reward modeling naturally translates to a target that directly aligns the LLM,
which we call \underline{R}egularized \underline{P}reference \underline{O}ptimization (RPO; Algorithm~\ref{alg1}).
The objective of RPO features a simple weighted combination of two losses: 
\begin{align}
    \textcolor{black}{\text{RPO objective}} \,=\, \textcolor{blue!55}{\text{Preference optimization loss}} \,+\, \textcolor{purple!55}{\text{Imitation (SFT) loss}}.
\end{align}
Here the \textcolor{blue!55}{\text{Preference optimization loss}} coincides with the DPO \citep{rafailov2024direct} objective, which tends to optimize the LLM towards maximizing the underlying true reward.
The \textcolor{purple!55}{\text{Imitation (SFT) loss}} explicitly supervises the LLM to mimic the responses from a proper distribution that is well covered by the dataset. 
The choice of the distribution is guided and justified by our theory of Algorithm~\ref{alg0}, but can also be flexibly adapted in practice, e.g., the preferred response in the dataset, or the responses of the initial model. 

We highlight that the \textcolor{purple!55}{\text{Imitation (SFT) loss}} serves as an important term to mitigate overoptimization.
Even though the original DPO objective has already involved a KL regularization between the tuned LLM and the initial LLM, is not enough to prevent overoptimization. 
As elaborated in Section~\ref{sec: practical algorithm}, the KL-regularization weight of the DPO objective could only control the scale of the gradient per training example, while the RPO objective can further modify the gradient direction.
Calling back to the theoretical Algorithm~\ref{alg0}, such a modification of gradient direction originates from the reward penalty in the adversarial objective for the reward model. 
This modification, as we expose in our theoretical analysis, helps to mitigate overoptimization.
\emph{Thus, incorporating SFT loss in RLHF gives you a regularizer that provably mitigates overoptimization.}

\paragraph{Empirical evaluations.} 
Following the training setup of two series of released chat models \texttt{Zephyr}-\texttt{7b-beta} (trained on Ultrafeedback dataset \citep{cui2023ultrafeedback} using DPO) and \texttt{Zephyr-7b-gemma} (trained on Argilla-DPO-Mix-7K dataset \citep{dpo-mix-7k} using DPO) \citep{tunstall2023zephyr}, we implement RPO for the beta series and gemma series respectively to show that: 
(i) RPO is a flexible plug-in module and can be applied to different reference models. 
(ii) RPO effectively alleviates overoptimization.
(iii) RPO consistently achieves better alignment performance than DPO in in-data distribution. 
(iv) RPO can also achieve consistently better performance in standard LLM benchmarks like MT-bench and AlpacaEval 2.0, which shows its potential of mitigating overoptimization for better alignment performance, justifying our theory.

\subsection{Related Works}\label{subsec: related works}

In the following, we relate our work to recent lines of RLHF research on both theory and practice sides. 
We also review related works on reward hacking and overoptimization in RLHF.

\paragraph{RLHF: algorithm design.} 

The technique of RLHF \citep{christiano2017deep, ziegler2019fine, ouyang2022training, bai2022training} has recently demonstrated its great importance in building the state-of-the-art LLMs, including ChatGPT \citep{achiam2023gpt}, Gemini \citep{team2023gemini}, Claude \citep{claude}. 
In the RLHF pipeline therein, the LLM is fine-tuned towards maximizing a learned reward model using the RL algorithm Proximal Policy Optimization (PPO; \cite{schulman2017proximal}).
Meanwhile, PPO-style algorithm is also known for its instability, sample-inefficiency, and especially, a high demand for proper hyperparameter tuning \citep{engstrom2020implementation}.
This thus casts prohibitive  computational cost to make the most effectiveness of PPO-based RLHF methods to align LLMs, especially for the open-source community.

Given that, further research on RLHF has explored various alternatives to PPO-based methods, with the most popular approach being the direct preference optimization method \citep{zhao2023slic, rafailov2024direct}, which skips the reward learning phase and directly optimize the LLM to align it with the human preference. 
Our practical implementation (RPO) also harnesses the wisdom of reward-LLM equivalence to avoid explicit reward learning followed by PPO training.

Besides the original DPO algorithm \citep{rafailov2024direct}, ever since it popularizing the direct preference learning style method, variants of the direct preference learning approach are proposed, including but not limited to \cite{liu2023statistical, azar2023general, xiong2023gibbs, tang2024generalized, ji2024towards, ye2024theoretical, pal2024smaug, hong2024orpo, rosset2024direct, liang2024robust, zhang2024negative, tajwar2024preference, wu2024self}.
Each of them aims to address further challenges of direct preference learning from varying perspectives.
Specifically, the algorithm proposed by \cite{pal2024smaug, hong2024orpo} share similar algorithmic components as RPO proposed in this work. 
Both work consider SFT style regularization during preference optimization. 
However, theoretical understanding of how SFT loss can help alignment remains unknown. 
In contrast, we provide theoretical justifications to the SFT loss as an implicit adversarial regularizer that provably mitigates overoptimization in preference learning.

\paragraph{RLHF: theoretical investigation.}

Initiated from the literature of dueling bandits and dueling RL \citep{yue2012k, bengs2021preference,pacchiano2021dueling}, recent success of RLHF in fine-tuning LLMs also motivates a long line of research to investigate the theoretical foundations of RLHF under different settings
\citep{pmlr-v162-chen22ag, zhu2023principled, zhan2023provable, zhan2023query, wang2024rlhf, li2023reinforcement, xiong2023gibbs, ye2024theoretical, du2024exploration, zhong2024dpo}, aiming to propose provably sample-efficient algorithms to learn a human-reward-maximizing policy from human preference signals.
Our theoretical study of RLHF falls into the paradigm of offline learning from a pre-collected preference dataset, and is most related to \cite{zhu2023principled, zhan2023provable, li2023reinforcement, xiong2023gibbs, ye2024theoretical}. 
In this setup, the main challenge is to address the overoptimization issues due to human reward uncertainty and distributional shifts when only a fixed dataset is available.
In the sequel, 
we compare our work with them in more detail.

Existing theoretical work on provably sample-efficient offline RLHF typically suffers from two drawbacks: they are either restricted to the linear function approximations setting \citep{zhu2023principled, xiong2023gibbs} which is far from the practical situations, or are generally unable to be implemented in the LLM experiments. 
Typically, to encompass the pessimistic principle in the face of uncertainty, the existing literature proposes to return the optimal policy against either an estimated reward model plus 
a structure-aware reward uncertainty penalty \citep{xiong2023gibbs} or the most pessimistic reward model inside a confidence region \citep{zhu2023principled, zhan2023provable}. 
Both of these two types of method involve intractable components for implementation and needs for additional algorithmic design to approximate the theoretical algorithm in practice.
In contrast, our theory works in the context of general function approximations while being friendly to be implemented.
Finally, we remark that, while our study focuses on the standard Bradley-Terry model of human preference with general reward function approximations, the work of \cite{ye2024theoretical} further considers a general human preference model.
But it remains unknown how their algorithms can be efficiently implemented in practice. 
It serves as an interesting direction to extend our technique to RLHF with general reward model and device new practical algorithms.

Finally, we mention that the algorithm design of RPO is also related to the ``pessimism'' principle in the standard offline RL literature.
It proposes to maintain a pessimistic estimate of the policy values or constrain the policy not to take unseen actions in the data to handle the challenge of the insufficient coverage of the dataset, e.g., \cite{jin2021pessimism,uehara2021pessimistic,xie2021bellman,xie2021policy,rashidinejad2021bridging,yin2021towards,fu2022offline,xiong2022nearly,liu2022welfare,liu2022learning, shi2022pessimistic,zhan2022offline,lu2022pessimism,rashidinejad2022optimal,shi2024distributionally, blanchet2024double}.
In contrast, we consider the offline RLHF problem and the techniques to obtain the objective of our algorithm along with its sample complexity analysis are new and different from these prior works on standard offline RL.

\paragraph{Reward hacking and overoptimization in RLHF for LLM.}
As is discussed in the introduction, the challenge of reward hacking or overoptimization may prevent the successful alignment of LLMs, degenerating the performance of an LLM because of maximizing an imperfect, overfitted, and misgeneralized proxy reward learned from the finite data \citep{michaud2020understanding, tien2022causal, gao2023scaling, casper2023open}.
Efforts have been made to mitigate this fundamental issue through the perspective of theory, e.g.,   \cite{zhu2023principled, xiong2023gibbs, zhu2024iterative}, and practice, e.g., \cite{ coste2023reward, eisenstein2023helping, dong2023raft, 
moskovitz2023confronting, zhang2024overcoming, rita2024countering, pal2024smaug, hong2024orpo}.
Our approach starts from the theoretical insights of handling inherent uncertainty in learning human preference from finite data, while being suprisingly easy to implement in practice.

\section{Preliminaries of RLHF}\label{sec: preliminaries}

In this section, we introduce the mathematical framework of studying RLHF for alignining LLMs. 
We adopt the framework of offline contextual bandits \citep{ouyang2022training, zhu2023principled, zhu2024iterative}, where we identify the contextual space $\cX$ as the space of prompts, and we identify the action space $\cA$ as the space of responses. 
An LLM, which is defined as a policy $\pi(\cdot|\cdot):\cX\mapsto\Delta(\cA)$, can take a prompt $x\in\cX$ as its input and output a response $a\in\cA$ obeying $a\sim \pi(\cdot|x)$.

\paragraph{Preference model.} Given any reward function $r:\cX\times\cA\mapsto\mathbb{R}$ belonging to certain reward class $\mathcal{R}$ which represents the ``human's rating" of LLM responses given some prompts, we consider the Bradley-Terry model \citep{bradley1952rank} of human preference. 
That is, given a prompt $x\in\cX$ and two responses $a^1, a^0\in\cA$, the probability of $a^1$ being preferred to $a^0$ (denoted by $y=1$, otherwise $y=0$) is given by
\begin{align}
    \mathbb{P}_r(y=1|x,a^1,a^0) = \frac{\exp(r(x,a^1))}{\exp(r(x,a^1))+ \exp(r(x,a^0))}= \sigma\big(r(x,a^1) - r(x,a^0)\big),\label{eq: bt model}
\end{align}
where $\sigma(z) = 1/(1+\exp(-z))$ is the sigmoid function.
For simplicity of future discussion, we explicitly write out the dependence of the preference probability $\mathbb{P}_r(\cdot)$ on the reward model $r\in\mathcal{R}$.
In the section of theory, i.e., Section~\ref{sec: theory}, we specify the assumptions on the reward model class $\cR$.

\paragraph{Learning protocol.} 
Typically in an LLM training pipeline, the RLHF phase starts from certain reference policy $\pi^{\mathrm{ref}}$ which is obtained by pretraining and then supervised fine-tuning (SFT). 
Then RLHF aligns the LLM based on certain human preference data. 
In this work, we consider offline RLHF setup where the LLM is aligned using a fixed offline preference dataset $\cD$. It consists of $N$ i.i.d. tuples in the form of
\begin{align}
    \cD = \big\{(x_i,a^1_i,a^0_i,y_i)\big\}_{i=1}^N.
\end{align}
Here the prompt $x_i$ and the response $a_i^1, a_i^0$ are distributed according to $(x,a^1,a^0)\sim \mu_{\cD}(\cdot)$, and conditioning on $(x_i,a_i^1,a_i^0)$, $y_i$ is distributed according to \eqref{eq: bt model} for an underlying true (but unknown) reward model $r^{\star}\in\mathcal{R}$.

\paragraph{Performance metric.} 
The target of RLHF is to align an LLM, or equivalently, to learn a policy $\pi$, so as to maximize the expected true reward $r^{\star}$.
Correspondingly, we define the value function of a policy $\pi$ as 
\begin{align}
    J(\pi) = \mathbb{E}_{x\sim d_0,a\sim \pi(\cdot|x)}\big[r^{\star}(x,a)\big].\label{eq: value function}
\end{align}
Here we allow the prompt distribution $d_0(\cdot)$ to be different from that of the offline dataset distribution $\mu_{\cD}(\cdot)$, but is assumed to be known.
In the meanwhile, we consider the policies that share the same support as the reference policy $\pi^{\mathrm{ref}}$, that is, we take a policy class $\Pi$ as
\begin{align}
    \Pi = \Big\{\pi:\cX\mapsto\Delta(\cA)\,\Big|\, \mathrm{Supp}(\pi(\cdot|x))\subseteq\mathrm{Supp}(\pi^{\mathrm{ref}}(\cdot|x)),\,\,\forall x\in\cX\Big\}.\label{eq: policy class}
\end{align}
The performance gap of a learned policy $\widehat{\pi}\in\Pi$ with respect to any other policy $\pi\in\Pi$ is measured as 
\begin{align}
    \mathrm{Gap}^\pi(\widehat{\pi}) = J(\pi) - J(\widehat{\pi}).\label{eq: gap}
\end{align}
Our goal is to propose a sample-efficient and also implementation-friendly algorithm to learn a policy $\widehat{\pi}\in\Pi$ that can compete with a given policy $\pi\in\Pi$ in terms of $\mathrm{Gap}^{\pi}(\widehat{\pi})\leq \varepsilon$, using a number of samples polynomial in $1/\varepsilon$ and logarithmic in the complexity of the reward class $\mathcal{R}$ and the failure probability inversed $1/\delta$. 

We do not assume parametric forms of the reward class $\mathcal{R}$ \citep{zhu2023principled, xiong2023gibbs}, and thus we fall into the paradigm of general function approximations.

\section{A Theory-motivated Objective}\label{sec: theory-motivated objective}

Our method seeks to find the best policy $\widehat{\pi}$ against an \emph{adversarially} chosen reward model $\widehat{r}_{\mathrm{adv}}$ that minimizes a weighted sum of its expected value and the maximum likelihood estimation (MLE) loss.
Intuitively, such a reward model can prevent the overoptimization issue by taking its own value into account when minimizing the MLE loss. 
Since the reward value is also minimized when minimizing the sum, this method prevents the misleadingly high reward caused by the uncertainty due to finite data.
Formally, given two hyperparameters $\beta, \eta>0$ and a ``baseline policy'' $\pi^{\mathrm{base}}$, we define
\begin{align}
    T_{\beta, \eta}^{\mathrm{adv}}(\pi) = \min_{r\in \mathcal{R}}\left\{\eta\cdot\mathbb{E}_{\substack{x\sim d_0, a^1\sim\pi(\cdot|x),\\ a^0\sim \pi^{\mathrm{base}}(\cdot|x)}}\Big[r(x,a^1) - r(x,a^0) -\beta\cdot \mathrm{KL}\big(\pi(\cdot|x)\|\pi^{\mathrm{ref}}(\cdot|x)\big)\Big] +  {\cL}_{\cD}(r)\right\},
\end{align}
where the loss function $\cL_{\cD}(\cdot)$ is the average negative log-likelihood function of the BT model \eqref{eq: bt model} (and here it becomes the cross-entropy loss) over the preference dataset $\cD$, defined as 
\begin{align}\label{eq: L D}
    \cL_{\cD}(r) = -\widehat{\EE}_{\cD}\bigg[y_i \log \big(\sigma\big(r(x_i,a^1_i) - r(x_i,a^0_i)\big)\big) + (1-y_i)\log \big(\sigma\big(r(x_i,a^0_i) - r(x_i,a^1_i)\big)\big)\bigg].
\end{align}
As we can see, $T_{\beta, \eta}^{\mathrm{adv}}(\pi)$ is the minimum value of a weighted sum of the MLE loss and the expected reward value of $\pi$, but with two important modifications that we explain in the following.

\emph{Firstly}, we subtract another expected reward of certain policy $\pi^{\mathrm{base}}$. 
This is because the BT model \eqref{eq: bt model} essentially only uses the reward differences to define the preference probabilities. 
As a result, the data can only reveal information of the differences between the true reward $r^{\star}$ of different responses \citep{zhan2023provable}. 
Accordingly, we subtract a baseline expected reward value to match this observation. 
The choice of the baseline policy is discussed in the theory part (Section~\ref{sec: theory}) and experiments (Section~\ref{sec: experiment}).

\emph{Secondly}, we subtract a KL divergence between $\pi$ and $\pi^{\mathrm{ref}}$ from the expected reward, weighted by the coefficient $\beta>0$.  
Such a term is for practical considerations that would be explained in Sections~\ref{sec: practical algorithm} and \ref{subsec: minimax}. 
We note that the KL regularized reward is commonly adopted in RLHF practice to ensure the learned policy is not far away from the reference policy \citep{ouyang2022training, xiong2023gibbs}.

\begin{algorithm}[t]
	\caption{Theoretical Algorithm: Maximin Objective}
	\label{alg0}
	\begin{algorithmic}[1]
	\STATE \textbf{Input}: Preference dataset $\mathcal{D}$, parameters $\beta,\eta>0$, reference policy $\pi^{\mathrm{ref}}$, baseline policy $\pi^{\mathrm{base}}$. 
    \STATE \textbf{Output}: Policy $\widehat{\pi}$ given by \eqref{eq: target max min} with the cross-entropy loss function $\cL_{\cD}$ defined in \eqref{eq: L D}.
	\end{algorithmic}
\end{algorithm}

Finally, the overall algorithm design (Algorithm~\ref{alg0}) is to output the policy that maximizes $T_{\beta, \eta}^{\mathrm{adv}}(\pi)$, i.e., $\widehat{\pi} \in \argmax_{\pi\in\Pi}T_{\beta, \eta}^{\mathrm{adv}}(\pi)$, which gives the following theoretical target: 
\begin{tcolorbox}[ams align, colback=orange!5, boxrule=0.15mm]
    \!\!\!\!\!\!\!\!\widehat{\pi}\in \argmax_{\pi\in\Pi}\,\min_{r\in \mathcal{R}}\left\{\eta\cdot\mathbb{E}_{\substack{x\sim d_0, a^1\sim\pi(\cdot|x),\\ a^0\sim \pi^{\mathrm{base}}(\cdot|x)}}\Big[r(x,a^1) - r(x,a^0) -\beta \cdot\mathrm{KL}\big(\pi(\cdot|x)\|\pi^{\mathrm{ref}}(\cdot|x)\big)\Big] +  {\cL}_{\cD}(r)\right\}.\label{eq: target max min}
\end{tcolorbox} 
Given the form of \eqref{eq: target max min}, we name it the \textcolor{orange!70}{\emph{maximin}} objective in the sequel. 
Upon seeing \eqref{eq: target max min}, one might be arguing that such a theory-motivated objective seems hard to implement in practice. 
Nevertheless, in the coming Section~\ref{sec: practical algorithm}, we demonstrate that the maximin objective \eqref{eq: target max min} adopts an easy-to-implement equivalent form, allowing us to design a practical algorithm for aligning LLMs.

\begin{remark}[Comparison with \cite{zhan2023provable}]
    We remark that in the work of \cite{zhan2023provable}, they also mentioned a maximin object similar to \eqref{eq: target max min} for offline preference-based RL as a complementary to their theoretical algorithm. 
    However, the sample complexity of the maximin algorithm they presented is unknown. 
    Furthermore, our objective \eqref{eq: target max min} features another KL-regularization term, which is essential for the proposal of our new practical algorithm design for aligning LLM in Section~\ref{sec: practical algorithm}.
\end{remark}

\begin{remark}[Comparison with \cite{xiong2023gibbs}]
    Another theoretical work on RLHF \citep{xiong2023gibbs} explicitly models the KL-regularization between the target policy and the reference policy in the learning objective, referred to as the KL-regularized contextual bandit.
    This means that their metric becomes the KL-regularized expected reward.
    In contrast, here we put the KL-regularization as a component of our algorithm design, but we still keep the metric as the expected reward \eqref{eq: value function}. 
    Therefore our theory in Section~\ref{subsec: sample complexity} directly reveals how the learned policy performs in terms of the expected reward compared to any given target policy.
\end{remark}

\section{An Equivalent and Implementation-friendly Objective}\label{sec: practical algorithm}

In this section, we propose another \emph{minimax}-style objective that is equivalent to the maximin objective \eqref{eq: target max min}. 
Based on the minimax objective, we propose a new LLM aligning algorithm called \underline{R}egularized  \underline{P}reference \underline{O}ptimization (RPO). 
It draws inspirations from the reparametrization technique originated in Direct Preference Optimization (DPO) \citep{rafailov2024direct} and goes beyond to further address the overoptimization issue in offline RLHF by incorprating an SFT loss as an explicit adversarial regularizer.

\paragraph{An equivalent minimax objective.}
If the reward model class $\mathcal{R}$ satisfies certain regularity conditions, which we discuss in detail in Section~\ref{subsec: minimax}, the minimax theorem holds: solving the \textcolor{orange!70}{\emph{maximin}} objective \eqref{eq: target max min} is \emph{equivalent} to solving a \textcolor{purple!60}{\emph{minimax}} target, given by
\begin{align}\label{eq: target min max}
    \min_{r\in \mathcal{R}}\,\max_{\pi\in\Pi}\left\{\eta\cdot \mathbb{E}_{\substack{x\sim d_0, a^1\sim\pi(\cdot|x),\\ a^0\sim \pi^{\mathrm{base}}(\cdot|x)}}\Big[r(x,a^1) - r(x,a^0) -\beta \cdot \mathrm{KL}\big(\pi(\cdot|x)\|\pi^{\mathrm{ref}}(\cdot|x)\big)\Big] +  {\cL}_{\cD}(r)\right\}.
\end{align}
Such a minimax formulation \eqref{eq: target min max} is the starting point of our practical algorithm. 
The magic of \eqref{eq: target min max} is that the inner maximization problem adopts a closed form solution, which further simplifies such an objective. 
To see this, note that given any reward model $r\in\cR$, the inner problem is equivalent to 
\begin{align}
    \max_{\pi\in\Pi}\bigg\{\mathbb{E}_{x\sim d_0, a\sim\pi(\cdot|x)}\Big[r(x,a) -\beta\cdot \mathrm{KL}\big(\pi(\cdot|x)\|\pi^{\mathrm{ref}}(\cdot|x)\big)\Big]\bigg\}.\label{eq: kl-regularized reward}
\end{align}
It has been well established that the policy that maximizes the KL-regularized expected reward \eqref{eq: kl-regularized reward} has a closed form solution. 
Due to its importance, we present it as the following lemma.
\begin{lemma}[Oracle optimal KL-regularized policy]\label{lem: optimal policy}
    Given any reward model $r\in\mathcal{R}$, the optimal policy $\pi_r$ to the maximization problem \eqref{eq: kl-regularized reward} is given by 
\begin{align}
    \pi_r(\cdot|x)= \frac{1}{Z_r(x)}\cdot \pi^{\mathrm{ref}}(\cdot|x)\cdot \exp\left(\beta^{-1} r(x,\cdot)\right),\,\, Z_r(x) = \int_{a\in\cA} \exp\left(\beta^{-1} r(x,a)\right)\mathrm{d}\pi^{\mathrm{ref}}(a|x),
\end{align}
and correspondingly the optimal value of \eqref{eq: kl-regularized reward} is given by $\eqref{eq: kl-regularized reward} =  \mathbb{E}_{x\sim d_0}[\beta\cdot \log(Z_r(x))]$.
\end{lemma}

\begin{proof}[Proof of Lemma~\ref{lem: optimal policy}]
    This follows directly from Proposition 7.16 and Theorem 15.3 of \cite{zhang2023mathematical}.
\end{proof}

Specifically, by Lemma~\ref{lem: optimal policy}, 
we can solve the inner maximization problem in \eqref{eq: 
target min max} and obtain that
\begin{align}
    \eqref{eq: target min max} = \min_{r\in \mathcal{R}}\bigg\{\eta\cdot \mathbb{E}_{\substack{x\sim d_0,  a^0\sim \pi^{\mathrm{base}}(\cdot|x)}}\Big[ - r(x,a^0) + \beta\cdot\log\left(Z_r(x)\right)\Big] +  {\cL}_{\cD}(r)\bigg\}.
\end{align}
Furthermore, from Lemma~\ref{lem: optimal policy}, one immediately see that given any reward model $r\in\mathcal{R}$, we can reparameterize it via its corresponding optimal KL-regularized policy $\pi_r$ \citep{rafailov2024direct}, that is,
\begin{align}
    r(x,\cdot) = \beta\cdot\log\left(\frac{\pi_r(\cdot|x)}{\pi^{\mathrm{ref}}(\cdot|x)}\right) + \beta\cdot\log (Z_r(x)).\label{eq: reparametrize}
\end{align}
Taking \eqref{eq: reparametrize} back into \eqref{eq: target min max}, we are able to further simplify it as
\begin{align}
    \!\!\!\!\!\!\!\!\!\!\eqref{eq: target min max} 
     = \min_{r\in \mathcal{R}}\bigg\{\eta\cdot\mathbb{E}_{\substack{x\sim d_0, a^0\sim \pi^{\mathrm{base}}(\cdot|x)}}\Big[ -\beta\cdot\log(\pi_r(a^0|x))\Big] +  {\cL}_{\cD}\left(\beta\cdot\log\left(\frac{\pi_r(\cdot|\cdot)}{\pi^{\mathrm{ref}}(\cdot|\cdot)}\right)\right)\bigg\}.\label{eq: target min max theory final}
\end{align}
Thanks to the KL-regularization term in the original minimax objective \eqref{eq: target min max} (or equivalently, the maximin objective \eqref{eq: target max min}), we have the following theorem. 
It theoretically shows that the policy $\pi_{\widehat{r}}$ associated with the reward model $\widehat{r}$ solving \eqref{eq: target min max theory final} also solves the maximin target \eqref{eq: target max min} of the theoretical algorithm (Algorithm~\ref{alg0}) that enjoys finite-sample convergence guarantees.

\begin{theorem}[Equivalence between \textcolor{orange!70}{\emph{maximin}} and \textcolor{purple!60}{\emph{minimax}} algorithm (informal)]\label{prop: equivalence informal}
    Under certain regularity assumptions on $\cR$ and given $\eta,\beta>0$, solving the minimax objective \eqref{eq: target min max} via \eqref{eq: target min max theory final}, i.e., 
    \begin{align}
        \widehat{r} = \argmin_{r\in\cR}\bigg\{\eta\cdot\mathbb{E}_{\substack{x\sim d_0, a^0\sim \pi^{\mathrm{base}}(\cdot|x)}}\Big[ -\beta\cdot\log(\pi_r(a^0|x))\Big] +  {\cL}_{\cD}\left(\beta\cdot\log\left(\frac{\pi_r(\cdot|\cdot)}{\pi^{\mathrm{ref}}(\cdot|\cdot)}\right)\right)\bigg\}, 
    \end{align}
    then the corresponding optimal KL-regularized policy $\pi_{\widehat{r}}$ also solves the maximin objective \eqref{eq: target max min}.
\end{theorem}

\begin{proof}[Proof of Proposition~\ref{prop: equivalence informal}]
    Please see Section~\ref{subsec: minimax} for a formal statement and proof of Theorem~\ref{prop: equivalence informal}.
\end{proof}

\begin{algorithm}[t]
	\caption{Practical Algorithm: \underline{R}egularized \underline{P}reference \underline{O}ptimization (RPO)}
	\label{alg1}
	\begin{algorithmic}[1]
	\STATE \textbf{Input}: Preference dataset $\mathcal{D}$, parameters $\beta,\eta>0$, reference policy $\pi^{\mathrm{ref}}$, baseline policy $\pi^{\mathrm{base}}$.
    \STATE \textbf{Output}: Policy $\pi_{\widehat{\theta}}$ obtained by optimizing objective \eqref{eq: RPO}.
	\end{algorithmic}
\end{algorithm}

\noindent
\textbf{Regularized Preference Optimization.} 
Target \eqref{eq: target min max theory final} gives a quite simple objective to use in practice! 
Since \eqref{eq: target min max theory final} depends on $r\in\cR$ only through its corresponding optimal policy $\pi_r$, we formulate a minimization objective over a parameterized policy $\pi_{\theta}$, i.e., the LLM to be aligned, and directly optimize the parameters $\theta\in\Theta$.
More formally, the new RLHF objective becomes 
\begin{tcolorbox}[ams align, colback=pink!10, boxrule=0.15mm]
    \!\!\!\!\min_{\theta\in\Theta}\Bigg\{\cL_{\mathrm{RPO}}(\theta):=\eta\beta\cdot\underbrace{\mathbb{E}_{\substack{x\sim d_0, a^0\sim \pi^{\mathrm{base}}(\cdot|x)}}\Big[ -\log(\pi_{\theta}(a^0|x))\Big]}_{\displaystyle{\text{\textcolor{purple!55}{\text{Imitation (SFT) loss}}}}} +  
    \underbrace{\cL_{\cD}\left(\beta\cdot\log\left(\frac{\pi_{\theta}(\cdot|\cdot)}{\pi^{\mathrm{ref}}(\cdot|\cdot)}\right)\right)}_{\displaystyle\text{\textcolor{blue!55}{\text{Preference opt. loss}}}}\Bigg\}.\label{eq: RPO}
\end{tcolorbox}
In \eqref{eq: RPO}, the second term coincides with the objective of DPO algorithm \citep{rafailov2024direct} which optimizes the policy towards maximizing the underlying true reward, and the first term stands for a regularization term weighted by $\eta\cdot\beta$ which \emph{explicitly} regularizes the policy to imitate the baseline policy. 
Therefore, we name the resulting algorithm as Regularized Preference Optimization (RPO). 
We summarize it abstractly in Algorithm~\ref{alg1}. 
As for DPO, implementing RPO does not require to maintain a reward model $r$. 
Thus it is computationally more friendly compared to reward-based algorithms.

\paragraph{How does RPO improve DPO?} 
We illustrate the effect of the imitation loss by analyzing the gradient of the RPO target $\cL_{\mathrm{RPO}}(\theta)$ in \eqref{eq: RPO}. 
Notice that by \eqref{eq: RPO} we have
\begin{align}
    \vspace{-2mm}
    \nabla_\theta \cL_{\mathrm{RPO}}(\theta)& = \eta\beta\cdot\underbrace{\mathbb{E}_{\substack{x\sim d_0, a^0\sim \pi^{\mathrm{base}}(\cdot|x)}}\Big[ -\nabla_\theta \log(\pi_{\theta}(a^0|x))\Big]}_{\text{increase the alignment with the baseline policy}} 
    +\underbrace{\nabla_\theta {\cL}_{\mathrm{DPO}}(\theta)}_{\text{decrease the DPO Loss}},\label{eq: gradient rpo}
\end{align}
where the derivative of the DPO loss $\nabla_\theta {\cL}_{\mathrm{DPO}}(\theta)$ is given by the following,
\begin{align}
\nabla_\theta {\cL}_{\mathrm{DPO}}(\theta) = - \widehat{\mathbb{E}}_{\mathcal{D}}\bigg[\underbrace{\beta\cdot\sigma\big(\widehat{r}_\theta(x, a_{\mathrm{rej}})-\widehat{r}_\theta(x, a_{\mathrm{cho}})\big)}_{\text{gradient weight}}\!\cdot \Big(\nabla_\theta \log \pi_{\theta}(a_{\mathrm{cho}} | x)-\nabla_\theta \log \pi_{\theta}(a_{\mathrm{rej}} |x)\Big)\bigg].
\end{align}
For simplicity we denote $\widehat r_\theta(x,a) = \beta\cdot \log(\pi_\theta(a|x))/\log(\pi^\mathrm{ref}(a|x))$, and denote $a_{\mathrm{cho}}$ for the chosen response and $a_{\mathrm{rej}}$ for the rejected response.
Intuitively, RPO \eqref{eq: RPO} modifies the \textbf{gradient direction} of DPO to ensure the alignment with the baseline policy $\pi^\text{base}$, and the hyper-parameter $\eta$ controls the power of alignment.
In comparison, the hyper-parameter $\beta$ in DPO only controls the \textbf{gradient weight} when increasing the likelihood of $a_{\mathrm{cho}}$ and decreasing the likelihood $a_{\mathrm{rej}}$. 
In this perspective, the hyper-parameter $\beta$ only changes the scale of the gradient instead of the direction. By introducing $\eta$, we stabilize the training and reduce the side-effect of uncertain labels in data to prevent overoptimization.

\section{Theoretical Analysis}\label{sec: theory}

In this section, we present theoretical analysis of our proposed algorithms (Algorithm~\ref{alg0} and Algorithm~\ref{alg1}). 
In Section~\ref{subsec: sample complexity}, we first establish a finite-sample complexity result for the theoretical algorithm (Algorithm~\ref{alg0}) that adopts maximin objective \eqref{eq: target max min}.
In Section~\ref{subsec: minimax}, we formally provide an equivalence result between the maximin objective \eqref{eq: target max min} and the minimax objective \eqref{eq: target min max} that we use in practice, based on which we further obtain the sample complexity guarantee of our practical algorithm design (Algorithm~\ref{alg1}).
In Section~\ref{subsec: generalization}, we further extend our theory to guarantee the performance on new prompt distributions.

Through this theory section, we assume that the space of prompts and responses are compact subsets of high-dimensional spaces, that is, $\cX\subseteq\mathbb{R}^{d_\cX}$ and $\cA\subseteq\mathbb{R}^{d_{\cA}}$ for some $d_\cX, d_\cA\in\mathbb{N}_+$, and $\cX$ and $\cA$ are compact (closed and bounded). 
We take the policy class $\Pi$ as \eqref{eq: policy class}, that is,  
\begin{align}
    \Pi = \Big\{\pi:\cX\mapsto\Delta(\cA)\,\Big|\, \mathrm{Supp}(\pi(\cdot|x))\subseteq\mathrm{Supp}(\pi^{\mathrm{ref}}(\cdot|x)),\,\,\forall x\in\cX\Big\}.\label{eq: policy class theory}
\end{align}

\subsection{Establishing the Sample Complexity of Maximin Objective (Algorithm~\ref{alg0})}\label{subsec: sample complexity}

We first make the following two assumptions for the sample complexity analysis of Algorithm~\ref{alg0}.

\begin{assumption}[True reward model]\label{as: regularity}
    We assume that the true reward model $r^{\star}\in\cR$, and for any $r\in\cR$ and $(x,a)\in\cX\times\cA$, it holds that $r(x,a)\in[0,R]$.
\end{assumption}

\begin{assumption}[Partial coverage coefficient \citep{zhan2023provable}]\label{as: coverage}
    Given a policy $\pi\in\Pi$, the coverage coefficient of the offline dataset distribution $\mu_{\cD}$ w.r.t. reward model class $\cR$, policy $\pi$, and the baseline policy $\pi^{\mathrm{base}}$, denoted by $C_{\mu_{\cD}}(\cR; \pi, \pi^{\mathrm{base}})$, is defined as 
    \begin{align}
        \max\left\{0, \sup_{r\in\cR}\frac{\mathbb{E}_{x\sim d_0,a^1\sim \pi(\cdot|x),a^0\sim \pi^{\mathrm{base}}(\cdot|x)}\big[(r^{\star}(x,a^1) -r^{\star}(x,a^0)) - (r(x,a^1) -r(x,a^0))\big]}{\sqrt{\mathbb{E}_{(x,a^1,a^0)\sim \mu_{\cD}}\left[\big|(r^{\star}(x,a^1) -r^{\star}(x,a^0)) - (r(x,a^1) -r(x,a^0))\big|^2\right]}}\right\}.
    \end{align}
    We assume that $C_{\mu_{\cD}}(\cR; \pi, \pi^{\mathrm{base}})<+\infty$ for the policy $\pi$ to compete. 
\end{assumption}

Assumption~\ref{as: regularity} is standard in sample complexity analysis \citep{zhu2023principled, zhan2023provable, ye2024theoretical}. 
Assumption~\ref{as: coverage} characterizes how well the dataset $\cD$ covers the policy $\pi$ to compete. 
Intuitively, this can be understood from the fact that $C_{\mu_{\cD}}(\cR; \pi, \pi^{\mathrm{base}})$ is upper bounded by the density ratio 
$$C_{\mu_{\cD}}(\cR; \pi, \pi^{\mathrm{base}})\leq \left\|\frac{d_0\otimes \pi\otimes \pi^{\mathrm{base}}}{\mu_{\cD}}\right\|_{\infty} = \sup_{(x,a^1,a^0)\in\cS\times\cA\times\cA}\frac{d_0(x)\pi(a^1|x)\pi^{\mathrm{base}}(a^0|x)}{\mu_{\cD}(x,a^1,a^0)}.$$
In order for Algorithm~\ref{alg0} to achieve provable sample efficiency, we only require that $\cD$ covers the target policy $\pi$, a weak partial coverage style assumption for the sample complexity analysis \citep{zhu2023principled, zhan2023provable, xiong2023gibbs}.
To illustrate it, when calling back to Figure~\ref{fig: 1}, the data distribution therein well covers those nearly optimal responses under $r^{\star}$, but does not sufficiently cover the responses with low $r^{\star}$.

Under such a partial coverage data condition,
however, the human preference of responses $a\in\cA$ that are not well covered by the dataset $\cD$ can be poorly estimated, misguiding the policy $\widehat{\pi}$ to behave suboptimally if it is overoptimized (recall Figure~\ref{fig: 1}).
Fortunately, the following theorem shows that Algorithm~\ref{alg0} provably mitigates the overoptimization issue and achieves a finite-sample convergence of the suboptimality gap \eqref{eq: gap}.

\begin{theorem}[Suboptimality of Algorithm~\ref{alg0}]\label{thm:main}
    Taking the policy class $\Pi$ as \eqref{eq: policy class}, supposing that Assumptions \ref{as: regularity} and \ref{as: coverage} hold, and assuming that the reward model class $\cR$ has a finite $\varepsilon$-covering number under $\|\cdot\|_{\infty}$-norm $\cN_{\varepsilon}(\cR,\|\cdot\|_{\infty})<+\infty$ with $\varepsilon = (6\cdot(1+e^R)\cdot N)^{-1}$.
    Setting 
    \begin{align*}
       \eta =  \frac{1}{(1+\exp(R))^{2}}\cdot\sqrt{\frac{24\cdot\log\left(\cN_{\varepsilon}(\cR,\|\cdot\|_{\infty})/\delta\right)}{N}},\quad \beta = \frac{1}{\sqrt{N}}
    \end{align*} 
    in Algorithm~\ref{alg0}. 
    Then the output policy $\widehat \pi$ of Algorithm~\ref{alg0} satisfies that with probability at least $1-\delta$,
    \begin{align}
        \mathrm{Gap}^\pi(\widehat{\pi})\leq \frac{\sqrt{6}\cdot\big(1+\exp(R)\big)^2\cdot \Big(\big(C_{\mu_{\cD}}(\cR; \pi, \pi^{\mathrm{base}})\big)^2+1\Big) \cdot \iota + 4 \cdot \mathbb{E}_{x\sim d_0}\Big[\mathrm{KL}\big(\pi(\cdot|x)\|\pi^{\mathrm{ref}}(\cdot|x)\big) \Big]}{4\sqrt{N}},
    \end{align}
    where $\iota =  \sqrt{\log\left(\cN_{\varepsilon}(\cR,\|\cdot\|_{\infty})/\delta\right)}$ with $\varepsilon = (6\cdot(1+e^R)\cdot N)^{-1}$.
    Here, $N$ denotes the number of preference pairs in $\cD$, $R$ denotes the upper bound of the reward models, and the partial coverage coefficient $C_{\mu_{\cD}}(\cR; \pi, \pi^{\mathrm{base}})$ is defined in Assumption~\ref{as: coverage}.
\end{theorem}

\begin{proof}[Proof of Theorem~\ref{thm:main}]
    Please refer to Appendix~\ref{subsec: proof main} for a detailed proof of Theorem~\ref{thm:main}.
\end{proof}

\begin{remark}[Choice of the baseline policy $\pi^{\mathrm{base}}$]
    As is indicated by Assumption~\ref{as: coverage}, the least requirement is that $\pi^{\mathrm{base}}$ can be covered by the offline data distribution.
    For example, we can take $\pi^{\mathrm{base}}$ as the distribution of the preferred responses in the dataset.
    In this case, the SFT loss in the RPO objective explicitly regularizes the LLM to imitate the preferred responses.
    We choose this baseline policy in our experiments (Section~\ref{sec: experiment}). \label{remark:baseline}
\end{remark}

\subsection{Equivalence between Maximin and Minimax Objectives}\label{subsec: minimax}

In this section, we formally show that under certain regularity assumptions, the theoretical target (maximin objective \eqref{eq: target max min}) and the target for practical algorithm design (minimax objective \eqref{eq: target min max}) are equivalent.
This can naturally extend the sample complexity of Algorithm~\ref{alg0} (Section~\ref{subsec: sample complexity}) to that of minimax-based algorithms in Section~\ref{sec: practical algorithm}, providing the theoretical guarantee for our practical algorithm design (RPO).

First, for notational simplicity, we denote the optimization target we investigate in Sections~\ref{sec: theory-motivated objective} and \ref{sec: practical algorithm} as 
\begin{align}
    \phi(\pi, r):= \eta\cdot\mathbb{E}_{x\sim d_0, a^1\sim\pi(\cdot|x), a^0\sim \pi^{\mathrm{ref}}(\cdot|x)}\Big[r(x,a^1) - r(x,a^0) -\beta\cdot D_{\mathrm{KL}}\big(\pi(\cdot|x)\|\pi^{\mathrm{ref}}(\cdot|x)\big)\Big] +  {\cL}_{\cD}(r),\label{eq: phi theory}
\end{align}
for any $(\pi,r)\in\Pi\times\cR$.
Our result relies on the following assumptions on the reward model class $\cR$.
\begin{assumption}[Regularity of reward model class]\label{ass: regularity}
    We assume the following things on the reward model class $\cR$: (i) the space $\mathcal{R}$ is a compact topological space; (ii) the function $\phi$ in \eqref{eq: phi theory} is convex-like on $\mathcal{R}$, that is, for any $r_1,r_2\in\mathcal{R}$ and  $\alpha\in[0,1]$, there exists $r_3\in\mathcal{R}$ such that 
    \begin{align}
        \phi(\pi,r_3)\leq \alpha\cdot\phi(\pi,r_1)+(1-\alpha)\cdot\phi(\pi,r_2),\quad \forall \pi\in\Pi,\label{eq: convex-like phi}
    \end{align}
\end{assumption}

We note that by the definition, the convex-like property \eqref{eq: convex-like phi} not only depends on the function $\phi$, but also depends on the reward model class $\mathcal{R}$. 
As a special case, if $\mathcal{R}$ is convex, e.g., a linear model class \citep{zhu2023principled, xiong2023gibbs, zhu2024iterative} or more general the Lipschitz continuous model class $\cR$, we can directly obtain that the function $\phi(\pi,\cdot)$ is \emph{convex} over $\mathcal{R}$ (since the dependence on $r\in\cR$ is linear terms plus a convex loss $\cL_{\cD}$ of $r\in\cR$), which implies the convex-like property \eqref{eq: convex-like phi}.

Under Assumption~\ref{ass: regularity}, it holds that (Lemma~\ref{prop: equivalence})
\begin{align}
    \max_{\pi\in\Pi}\,\min_{r\in \mathcal{R}}\phi(\pi,r) = \min_{r\in \mathcal{R}}\,\max_{\pi\in\Pi}\phi(\pi,r).\label{eq: minimax}
\end{align}
Furthermore, thanks to the KL-divergence regularization in $\phi$ which intuitively makes $\phi$ ``strongly concave" over the policy $\pi$, \eqref{eq: minimax} can gives us the following stronger result.

\begin{theorem}[Formal statement of Theorem~\ref{prop: equivalence informal}]\label{thm: equivalence formal}
    For the policy class $\Pi$ defined in \eqref{eq: policy class theory} and the reward model class $\cR$ satisfying Assumption~\ref{ass: regularity}, consider the following policy defined as
    \begin{align}   \pi_{\widehat{r}}\in\argmax_{\pi\in\Pi}\phi(\widehat{r},\pi),\quad where\quad  \widehat{r} \in \argmin_{r\in \cR}\,\max_{\pi\in\Pi}\phi(\pi, r).\label{eq: equivalent formal 1}
    \end{align}
    Then the policy $\pi_{\widehat{r}}$ also satisfies the maximin objective \eqref{eq: target max min} of Algorithm~\ref{alg0}, that is, 
    \begin{align}
        \pi_{\widehat{r}} \in \argmax_{\pi\in\Pi}\,\min_{r\in\cR} \phi(\pi,r).
    \end{align}
\end{theorem}

\begin{proof}[Proof of Theorem~\ref{thm: equivalence formal}]
    Please refer for Appendix~\ref{subsec: proof equivalence formal} for a detailed proof of Proposition~\ref{thm: equivalence formal}.
\end{proof}

Theorem~\ref{thm: equivalence formal} shows that the optimal KL-regularized policy associated with the reward model solving the minimax objective \eqref{eq: target max min} also solves the maximin target for the policy (i.e., objective \eqref{eq: target min max} of Algorithm~\ref{alg0}). 
This further allows us to extend our theoretical guarantee of Algorithm~\ref{alg0} (Section~\ref{subsec: sample complexity}) to that of minimax-based algorithms,  justifying our practical algorithm design in Section~\ref{sec: practical algorithm}.

\begin{corollary}[Suboptimality of minimax-based algorithm]\label{cor: subopt rdpo}
    Take the policy class $\Pi$ in \eqref{eq: policy class theory} and the reward model class satisfying Assumption~\ref{ass: regularity}.
    Given any given policy $\pi$ to compete, if Assumption~\ref{as: coverage} holds for $\pi$, then under the same choice of $\eta$ and $\beta$ as in Theorem~\ref{thm:main}, 
    the policy $\pi_{\widehat{r}}$ defined in \eqref{eq: equivalent formal 1} satisfies that 
    \begin{align}
        \mathrm{Gap}^{\pi}(\pi_{\widehat{r}}) \leq \widetilde{\cO}\left(1/\sqrt{N}\right).
    \end{align}
    with probability at least $1-\delta$, where $\widetilde{\cO}(\cdot)$ hides the same factors as shown in Theorem~\ref{thm:main} for Algorithm~\ref{alg0}.
\end{corollary}

\begin{proof}[Proof of Corollary~\ref{cor: subopt rdpo}]
    This is a direct corollary of Theorem~\ref{thm:main} and Theorem~\ref{thm: equivalence formal}.
\end{proof}

\subsection{Generalization to New Prompt Distributions}\label{subsec: generalization}

In this section, we further generalize our previous theoretical analysis to guarantee the performance of the proposed algorithm on new prompt distributions different from the data or training distribution. 
Specifically, we would like to consider a new prompt distribution $d_1(\cdot)\in\Delta(\cX)$ that is different from $d_0$. 
This corresponds to testing the learned policy on a new set of prompts with a different distribution than the training prompts.

The following corollary demonstrates that, as long as the new prompt distribution $d_1$ is well covered by the distribution $d_0$, we still have a finite-sample convergence guarantee. 

\begin{corollary}[Generalization to new prompt distributions]\label{cor: context generalization}
    Under the same setups as in Corollary~\ref{cor: subopt rdpo}, we consider the policy to compete as the optimal policy $\pi^{\star}\in\Pi$, that is, $\pi^{\star} = \argmax_{\pi\in\Pi}\mathbb{E}_{a\sim \pi(\cdot|x)}[r^{\star}(x,a)]$ for any $x\in\cX$.
    Assume that the density ratio between the prompt distributions $d_1$ and $d_0$ are bounded, i.e., 
    \begin{align}
        C_{\infty}(d_0, d_1):=\sup_{x\in\cX}\,\frac{d_1(x)}{d_0(x)} < + \infty.\label{eq:coverage_x}
    \end{align}
    Then the following bound holds, for the policy $\pi_{\widehat{r}}$ (see Corollary~\ref{cor: subopt rdpo}), with probability at least $1-\delta$, 
    \begin{align}
        \mathbb{E}_{x\sim d_1, a\sim \pi^{\star}(\cdot|x)}[r^{\star}(x,a)] -  \mathbb{E}_{x\sim d_1, a\sim \pi_{\widehat{r}}(\cdot|x)}[r^{\star}(x,a)] \leq \widetilde{\cO}\left(\frac{C_{\infty}(d_0, d_1)}{\sqrt{N}}\right),
    \end{align}
    where $\widetilde{\cO}(\cdot)$ hides the same factors as shown in Theorem~\ref{thm:main} for Algorithm~\ref{alg0}.
\end{corollary}

\begin{proof}[Proof of Corollary~\ref{cor: context generalization}]
    Please refer to Appendix~\ref{sec: proof generalization} for a detailed proof of Corollary~\ref{cor: context generalization}.
\end{proof}

In our experiments (Section~\ref{sec: experiment}), we evaluate the performance of RPO-trained LLMs (trained using the Ultrafeedback Dataset) on two standard benchmarks, MT-Bench and AlpacaEval 2.0 (with a different prompt distribution than Ultrafeedback), which demonstrates the capability of RPO to adapt to new prompt distributions, as is indicated by the above Corollary~\ref{cor: context generalization}.

A key requirement of Corollary~\ref{cor: context generalization} is that the training prompt distribution well covers the new prompt distribution.
In the following, we use Principle Component Analysis (PCA) analysis to illustrate the prompt distribution of these two benchmarks as well as the training data.
We use \texttt{text-embedding-ada-002} provided by OpenAI to extract the embeddings of the prompts in the training dataset, MT-bench, and AlpacaEval 2.0, for the two models we train in experiments (the beta series and the gemma series).
Here, each embedding is a vector with a length of 1536. 
For each model, we use Singular Value Decomposition (SVD) on the matrix stacked by the embeddings of the training dataset and obtain the first two PCA axises $v_1$ and $v_2$. Here, unit vectors $v_1$ and $v_2$ have a length of  1536  and correspond to the first largest and the second largest singular value in SVD, respectively. Then, for each embedding $e$ in the training dataset, MT-Bench, and AlpacaEval 2.0, we use $(\langle e, v_1\rangle, \langle e, v_2\rangle)$ as the coordinate to draw a 2D scatter plot in Figures \ref{fig:pca-beta} and \ref{fig:pca_gema}. 
Results show that \eqref{eq:coverage_x} in Corollary \ref{cor: context generalization} approximately holds for both beta and gemma series on MT-Bench and AlpacaEval 2.0, which suggests the desired performance of RPO on these two benchmarks given by Corollary~\ref{cor: context generalization}. 

\begin{figure}[H]
\begin{minipage}{0.48\textwidth}\centering
       \includegraphics[width = 0.8\linewidth]{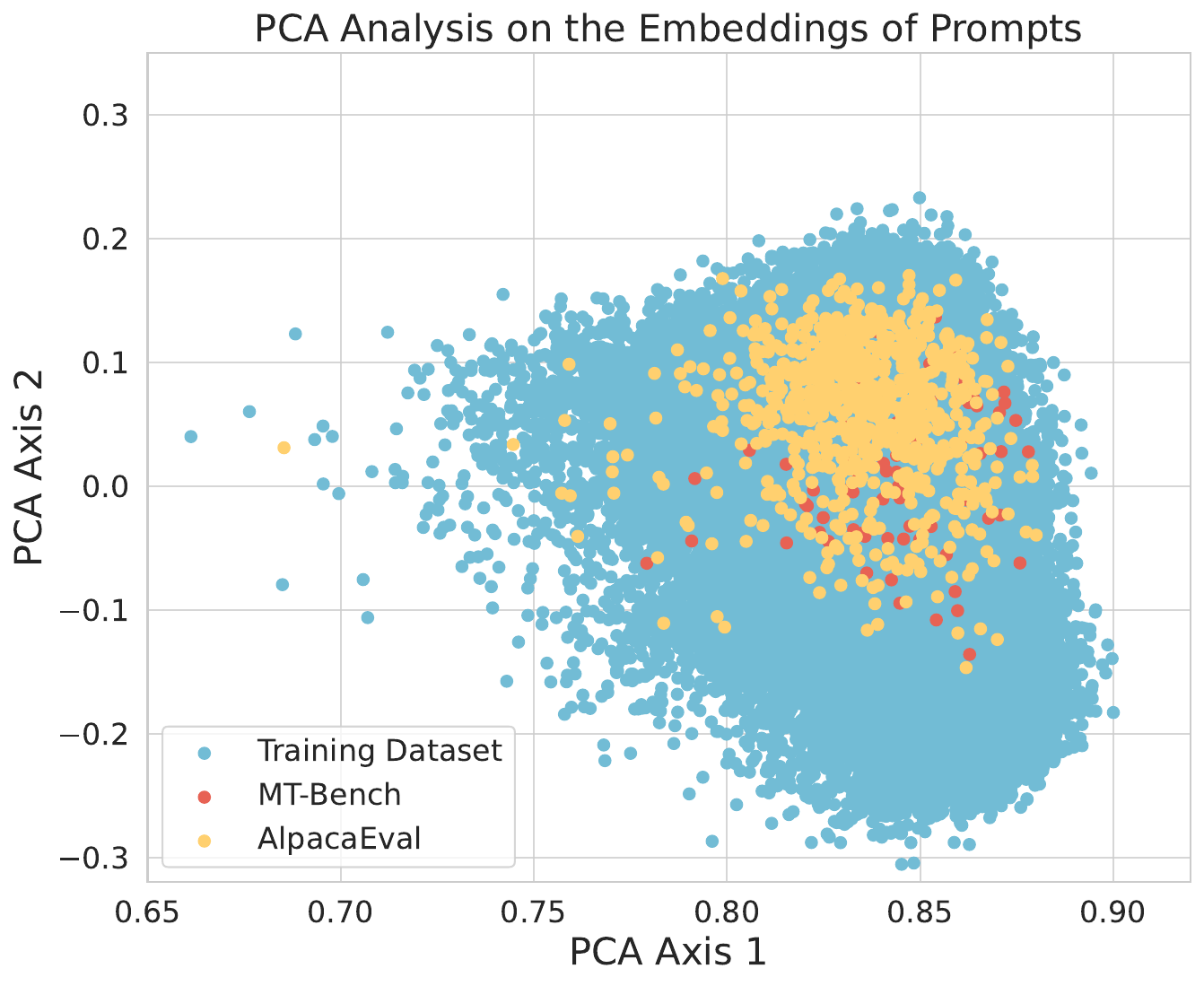}
    \caption{PCA of the embeddings of prompts in the training dataset of the beta series (Ultrafeedback), MT-Bench, and AlpacaEval 2.0. }\label{fig:pca-beta}
\end{minipage}\hspace{0.5cm}
\begin{minipage}{0.48\textwidth} 
\includegraphics[width = 0.8\linewidth]{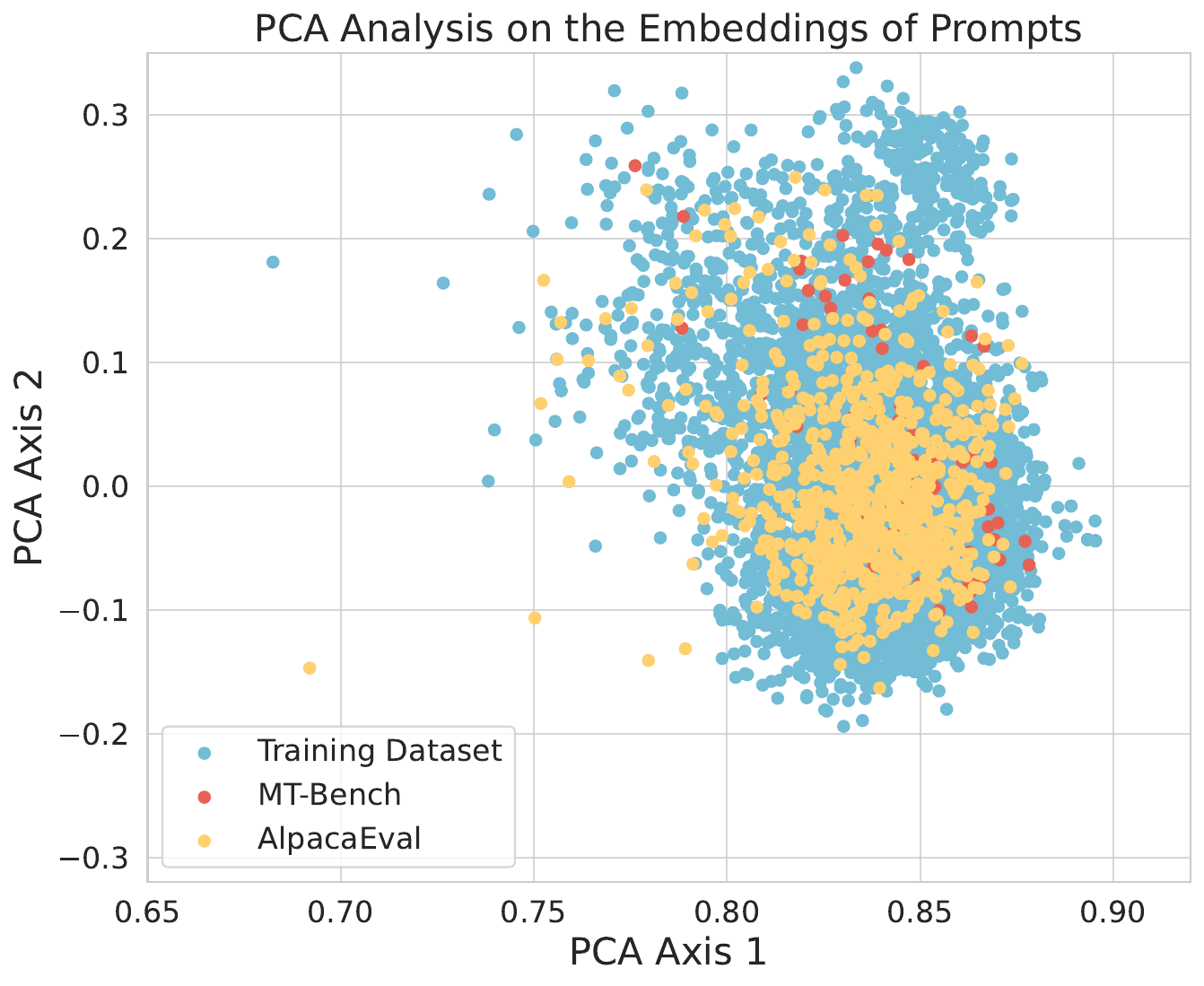}
       \caption{PCA of the embeddings of prompts in the training dataset of the gemma series (Argilla-DPO-Mix-7K), MT-Bench, and AlpacaEval 2.0. }\label{fig:pca_gema}
\end{minipage}
\end{figure}

\section{Experiments}\label{sec: experiment}
In this experiment section, we provide a detailed empirical analysis of RPO to highlight the following four key points: (i) RPO is a flexible plug-in module and can be applied to different reference models. (ii) RPO mitigates overoptimization in the training phase by giving more trust to the chosen responses in the preference dataset. (iii) As a justification of our theoretical analysis, RPO achieves better alignment performance than DPO in in-data distribution. (iv) RPO can also achieve consistently better performance in LLM benchmarks like MT-bench \citep{zheng2024judging} and AlpacaEval 2.0 \citep{dubois2024length}, which shows the potential of mitigating overoptimization for better generalization performance.

\paragraph{Experiment setup.} 
To show that RPO is a flexible plug-in module regardless of the reference model, we follow the training setup for two well-studied series of released chat models with around 7 billion parameters trained by DPO: \texttt{Zephyr}-\texttt{7b}-\texttt{beta} and \texttt{Zephyr}-\texttt{7b}-\texttt{gemma} \citep{tunstall2023zephyr} to implement RPO in beta and gemma series. Mirrored by their training configurations, we introduce how we select the reference model and the preference dataset for our training pipeline of these two series as follows. 
For the beta series, we use \texttt{mistral}-\texttt{7b}-\texttt{sft}-\texttt{beta} as the reference model $\pi^\text{ref}$. \texttt{mistral}-\texttt{7b}-\texttt{sft}-\texttt{beta} is a fine-tuned version of \texttt{Mistral-7b-v0.1} on the distilled version of the UltraChat dataset \citep{ding2023enhancing}, which contains approximately 200k examples of multi-turn
dialogues generated by GPT-3.5-TURBO. 
For the training preference dataset, we use the Ultrafeedback Dataset      \citep{cui2023ultrafeedback}, which consists of approximately 60k prompts. For the gemma series,  we use \texttt{zephyr}-\texttt{7b}-\texttt{gemma}-\texttt{sft}-\texttt{v0.1}  as our reference model $\pi^\text{ref}$. \texttt{zephyr}-\texttt{7b}-\texttt{gemma}-\texttt{sft}-\texttt{v0.1} is a fine-tuned version of \texttt{gemma-7b} on the Deita dataset \citep{liu2023makes}, which involves around 10k distilled SFT data.  
For the training preference dataset, we use the Argilla-DPO-Mix-7K Dataset \citep{dpo-mix-7k}, which is a mixture of multiple distilled public preference datasets.
For simplicity, we denote  Ref. (beta) as the reference model, DPO (beta) as the model trained by DPO, and RPO (beta) as the model trained by RPO, all for the beta series. We use the same notations for the gemma series. 

\paragraph{Practical implementation.}
According to Algorithm \ref{alg1} and as we discussed in Remark \ref{remark:baseline}, we implement RPO by adding an SFT loss (log probability of chosen responses in the preference dataset) to the original DPO loss. By comparing the evaluation performance on the test split of the training dataset, we select the hyperparameter $\eta$ as $0.005$ for both RPO (beta) and RPO (gemma). During the training of DPO and RPO, We keep the remaining hyperparameters including $\beta$, batch size, and learning rate to be the same for a fair comparison.  Please see Appendix \ref{app:training_details} for a detailed training configuration.

\paragraph{RPO alleviates overoptimization.} As mentioned in the introduction (Figure~\ref{fig: 1}), DPO is observed to have a significant and continual decrease in log probability on chosen responses \citep{hong2024orpo,rafailov2024r} during training and we regard it as the consequence of overoptimization. 
Implied by our theory, overoptimization could arise when the model maximizes its own proxy reward formed on the responses less covered by the data. 
Due to the overoptimization, the model tends to disprefer the chosen responses as they are away from the maximizers of the proxy reward despite that some chosen responses are highly preferred by humans. 
Consistent with our theoretical conclusion, we empirically find that RPO can indeed alleviate overoptimization in DPO. During the training phase of both beta and gemma series, we observe that the log probability given by the RPO-trained model is notably higher than that given by the DPO-trained model for the chosen responses, which are shown in Figures \ref{fig: 1}  and \ref{fig:training_stat_gemma}.  

\paragraph{RPO improves the alignment performance in in-data distribution.} For the in-data distribution evaluation, we select the 200 prompts (which are not used in the selection of paremeter $\eta$) in the test split of the training dataset to let the reference model, DPO, and RPO generate the response respectively. We choose GPT-4 to annotate the preference in the response pairs. Though we instruct GPT-4 to give an annotation among win, lose, and tie (please see the full prompt in Appendix \ref{app:eval}), GPT-4 might still provide undesired annotations. 
Therefore, we filter all the undesired annotations and collect 150 examples for evaluation. We report the pairwise win rate among Ref., RPO, and DPO in Table  \ref{tab:test} for both the beta and gemma series. To show a more illustrative comparison between DPO and RPO, we provide the barplot to report the number of pairwise examples annotated by GPT-4 in Figures \ref{fig:test_beta} and \ref{fig:test_gemma}. 
We observe that for both beta and gemma series, RPO has a better performance than DPO in terms of both RPO/DPO-SFT and RPO-DPO win rates.
The performance improvement matches our theoretical results in Corollary \ref{cor: subopt rdpo}, which shows the credit of the alleviation of overoptimization. 

\begin{table}[t]
    \centering
    \resizebox{1\textwidth}{!}{
    \begin{tabular}{cccc|cccc}    \toprule
       Win rate (\%) & RPO (beta) &  Ref. (beta)  & DPO (beta) &Win rate (\%) & RPO (gemma)  &  Ref. (gemma) & DPO (gemma)  \\
        \midrule
          RPO (beta) & 50.0 & \textcolor{blue}{\textbf{79.0}} &  \textcolor{blue}{\textbf{56.0}} & RPO (gemma)  & 50.0 & \textcolor{blue}{\textbf{71.7}}  &  \textcolor{blue}{\textbf{54.0}}\\
        Ref. (beta) & 21.0  & 50.0 & 22.7 & Ref. (gemma) & 28.3  & 50.0 & 32.7\\
        DPO (beta) & 44.0& 77.3  & 50.0 & DPO (gemma) &46.0 &  67.3 & 50.0 \\
        \bottomrule
    \end{tabular}}
    \caption{Pairwise win rate (left vs. right) among RPO-trained model, DPO-trained model, and the reference model. Annotated by GPT-4, evaluations of beta and gemma series are made on the 150 examples of the test split of the Ultrafeedback and the Argilla-DPO-Mix-7K dataset, respectively. }
    \label{tab:test}
\end{table}

\begin{figure*}[!t]
\vspace{-2mm}
\begin{minipage}{0.32\textwidth}\centering
       \includegraphics[width = \linewidth]{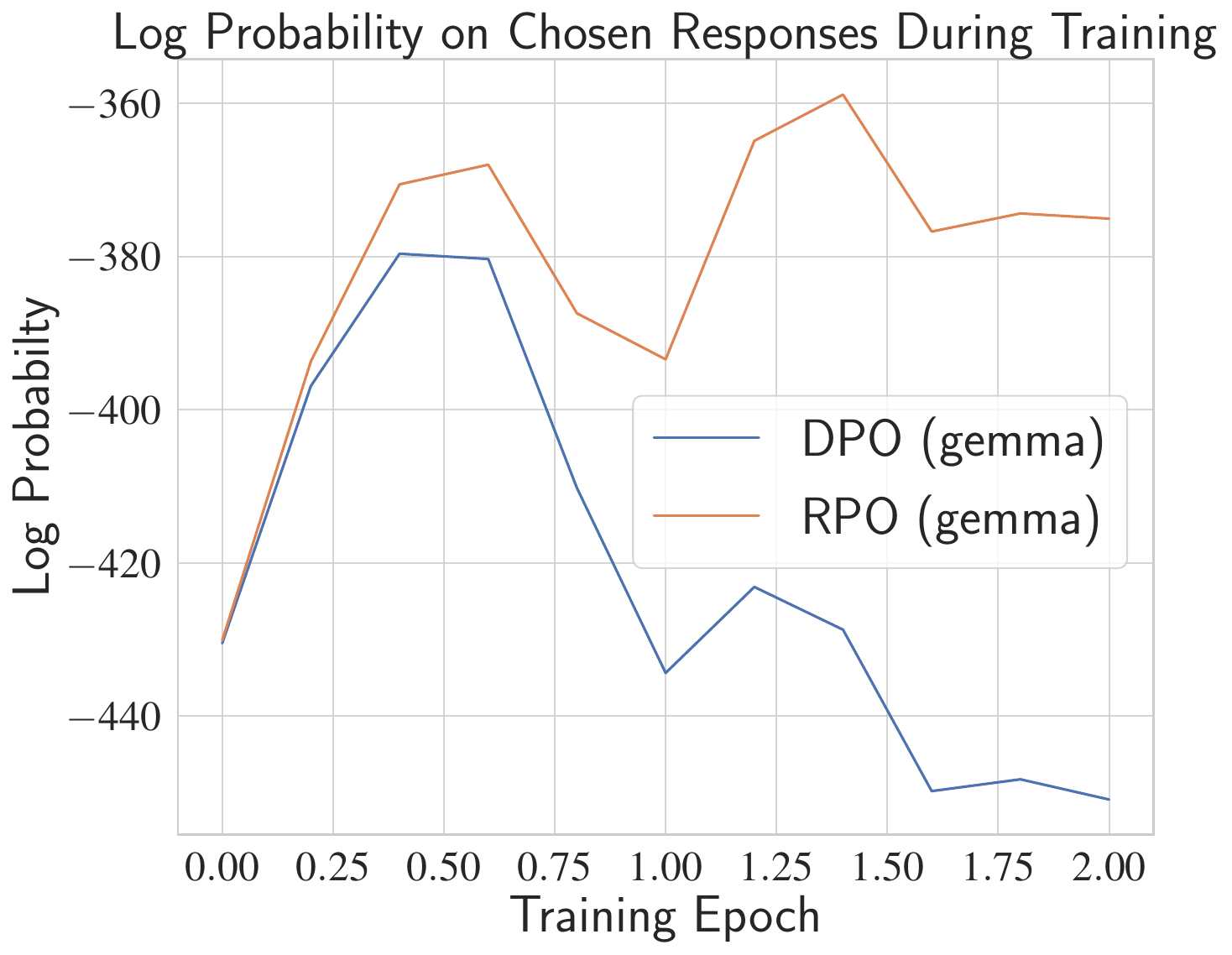}
       \caption{Log probability of chosen responses during the training of RPO (gemma) and DPO (gemma).}\label{fig:training_stat_gemma}
\end{minipage}\hspace{0.2cm}
\begin{minipage}{0.32\textwidth} 
\includegraphics[width = \linewidth]{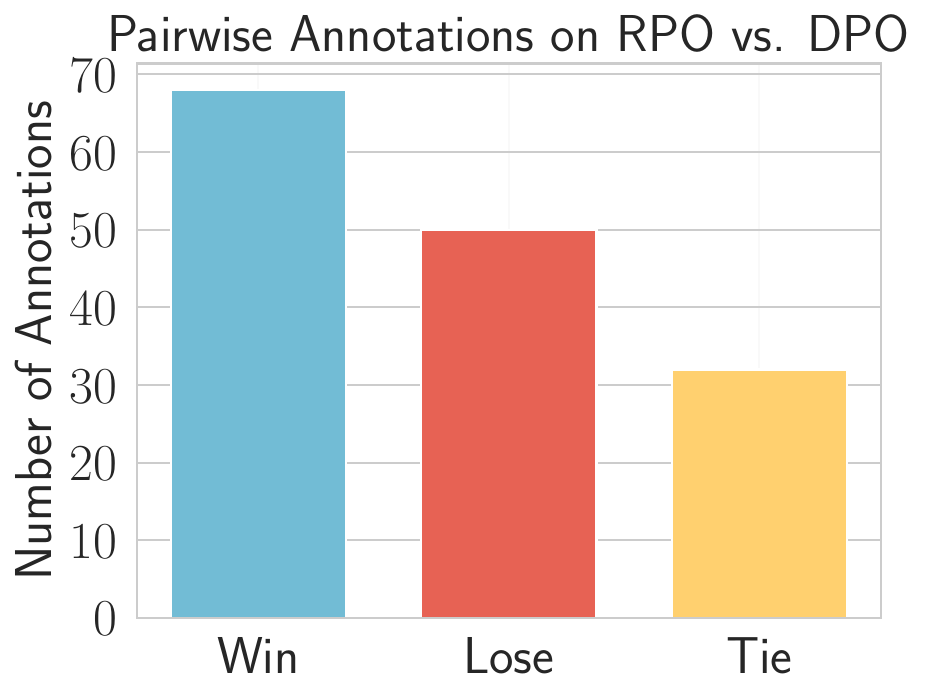}
       \caption{Pairwise annotations (by GPT-4) on RPO (beta) vs. DPO (beta) on the test split of the Ultrafeedback dataset.}\label{fig:test_beta}
\end{minipage}\hspace{0.2cm}
\begin{minipage}{0.32\textwidth}
\centering
       \includegraphics[width = \linewidth]{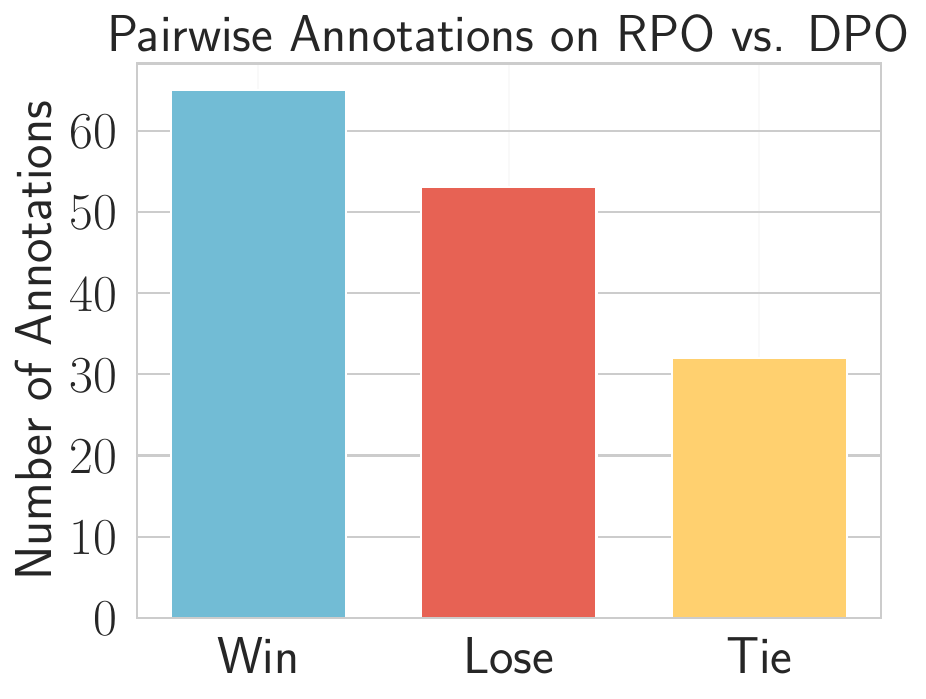}
       \caption{Pairwise annotations (by GPT-4) on RPO (gemma) vs. DPO (gemma) on the test split of the Argilla-DPO-Mix-7K dataset.}\label{fig:test_gemma}
\end{minipage}
\end{figure*}

\paragraph{RPO consistently improves the  benchmark performance.} We further evaluate the reference model, RPO-trained model, DPO-trained model, and the officially released DPO-trained model for both beta and gemma series in two standard LLM chat benchmarks: MT-bench \citep{zheng2024judging} and AlpacaEval 2.0 \citep{dubois2024length}.  MT-Bench is a multi-turn benchmark that contains 160 questions
across eight different domains of knowledge. 
The score for MT-Bench is evaluated by GPT-4 on a scale from 1 to 10. AlpacaEval 2.0 is a single-turn benchmark including 805 questions on different topics, mostly focused on helpfulness. The metrics of AlpacaEval 2.0 are the win rate and Length-Control (LC) win rate compared with GPT-4 Preview (11/06), where the annotator is also GPT-4 Preview (11/06) and LC win rate is proposed to mitigate the length bias of GPT-4. The results are summarized in Table \ref{tab:benchmark}, which shows that RPO consistently exceeds the performance of all the competitors (DPO, Reference model, and the officially released model trained by DPO) on MT-Bench and AlpacaEval 2.0. 
We also provide additional results on the pairwise win rate for these two benchmarks in Appendix \ref{app:more_results} to illustrate the performance improvement. The improvement can also be explained by the theoretical results in Corollary \ref{cor: context generalization}. 
Finally, we remark that RPO is a flexible plug-in module and can steadily improve the benchmark performance without changing the original training configuration or accessing extra preference data. 
This also sheds light on the potential of mitigating overoptimization for better alignment and generalization performance. 

\begin{table}[!t]
    \centering
    \resizebox{1\textwidth}{!}
    {
    \begin{tabular}{c  c  c c  cccc }
    \toprule
    \multirow{2}{*}{Model Name} & MT-Bench & \multicolumn{2}{c}{AlpacaEval 2.0} &\multirow{2}{*}{Model Name} & MT-Bench & \multicolumn{2}{c}{AlpacaEval 2.0}  \\
    \cline{3-4} \cline{7-8}
         & Score & LC win rate (\%)  & win rate (\%) &
         & Score & LC win rate (\%) & win rate (\%) \\
        \midrule
        RPO (beta)& \textcolor{blue}{\textbf{7.381}} & \textcolor{blue}{\textbf{23.28}}&  \textcolor{blue}{\textbf{21.01}} & RPO (gemma)& \textcolor{blue}{\textbf{7.916}} & \textcolor{blue}{\textbf{15.51}} & \textcolor{blue}{\textbf{13.85}} \\
       Ref. (beta) & 5.088 & 7.19& 4.69 &  Ref. (gemma) & 7.266& 8.35 &  4.61 \\
        DPO (beta)& 7.278 & 21.15 & 17.27&  DPO (gemma)& 7.688& 15.36 & 13.69 \\
        \texttt{zephyr-beta-7b} &7.200& 13.20& 10.99 &  \texttt{zephyr-gemma-7b} & 7.719& 14.78& 12.14\\
    \bottomrule
    \end{tabular}%
    }
    \caption{Results on MT-Bench scores and AlpacaEval 2.0. \texttt{zephyr}-\texttt{beta-7b} and \texttt{zephyr-gemma-7b} are the officially released models. win rates and Length-Control (LC) win rates in AlpacaEval 2.0 are evaluated by GPT-4 compared with GPT-4. }
    \label{tab:benchmark}
\end{table}

\paragraph{RPO also improves math, reasoning, and coding abilities.} 
In addition to the MT-Bench and AlpacaEval 2.0 benchmarks, we further use benchmarks on math, reasoning, and coding tasks for evaluations of the RPO algorithm. 
Specifically, we choose the Grade School Math 8K (GSM8K) \citep{cobbe2021training}, AI2 Reasoning Challenge (ARC) \citep{clark2018think}, and Mostly Basic Python Programming (MBPP) \citep{austin2021program} to measure the math, reasoning, and coding abilities of the model trained by RPO, respectively. 
We refer the readers to Appendix~\ref{subsec: further experiments} for the detailed setups and results of these experiments.

\section{Conclusions}
This work proposes a new algorithm that provably mitigates reward overoptimization in RLHF.
We establish its finite-sample convergence under a partial coverage style data condition, and provide an equivalent practical implementation,  RPO.
As a flexible plug-in module, RPO exhibits consistent improvement over the DPO baseline and effectively mitigates overoptimization.
Future works include extending our idea of theoretical algorithm design and analysis to the iterative RLHF setup where further preference data could be collected.
Also, since our practical algorithm RPO is a plug-in module that effectively mitigates overoptimization and improves alignment performance, it serves as an exciting direction to combine it with explorative preference data collecting mechanism in iterative RLHF to further boost the performance of LLM alignment.

\section*{Acknowledgement}
The authors would like to thank the anonymous reviewers for their helpful comments.
The authors would also like to thank Junyan Zhang for valuable discussions on the equivalence between the minimax and maxmin optimization.


\bibliography{reference}

\bibliographystyle{ims}

\newpage 

\appendix 

\section{Further Discussions}

In this section, we make several further discussions on the RPO algorithm and its theory

\paragraph{Discussions on the partial coverage assumption (Assumption~\ref{as: coverage}).}

A sufficient condition to make this partial coverage condition (Assumption~\ref{as: coverage}) hold is that the distribution of the offline dataset, which is $\mu_\mathcal{D}$, can well cover the joint distribution of $(a^1, a^0)\sim (\pi, \pi^{\mathrm{base}})$. Here to discuss focus on $\pi^{\mathrm{base}} = \pi^{\mathrm{chosen}}$ as we adopted in the experiment part.

First, we clarify that the offline dataset distribution $\mu_\mathcal{D}$ is not simply $(a^1, a^0)\sim(\pi^{\mathrm{unchosen}}, \pi^{\mathrm{chosen}})$, since according to our definition (see Section~\ref{sec: preliminaries}) whether $a^1$ or $a^0$ is chosen is random and is determined by $y \in{0, 1}$ obeying the BT model. 
Thus, $(a^1, a^0)\sim \mu_\mathcal{D}$ can be interpreted as a mixture of $(\pi^{\mathrm{unchosen}}, \pi^{\mathrm{chosen}})$ and $(\pi^{\mathrm{chosen}}, \pi^{\mathrm{unchosen}})$. 
This mixture probability would not be too small as long as the quality of $(a^1, a^0)$ does not vary too much, i.e., both of them are possible to be chosen, which is the case in practice. 
As a result, in the offline data distribution $(a^1,a^0)\sim \mu_{\mathcal{D}}$, both $a^1$ and $a^0$ partly comes from the chosen distribution $\pi^{\mathrm{chosen}}$.

Then in order for $\mu_\mathcal{D}$ to cover the joint distribution of $(a^1, a^0)\sim (\pi, \pi^{\mathrm{base}})$, it suffices to argue that $\pi^{\mathrm{chosen}}$ can cover the target policy $\pi$, which is then reduced back to the traditional coverage condition.
Thus our assumption essentially requires that $\pi^{\mathrm{chosen}}$ well covers and only needs to cover the target policy $\pi$. This coincides with the spirit of the minimal data assumption in offline RL theory, i.e., the so-called partial coverage condition.

\paragraph{On the relationship between observed chosen probability and reward overoptimization.}
First, we note that the actions and their chosen probabilities can be interpreted as a proxy of analyzing the underlying (estimated) reward model $\widehat{r}$  due to the representation $\pi_{\widehat{r}}(a|x)\propto\pi^{\mathrm{ref}}(a|x)\exp(\beta^{-1}\widehat{r}(x,a))$. Analyzing the (log) probabilities of the actions can be utilized to detect the mitigation of over-optimization, because according to the representation, an overestimated reward of a poor action would result in a higher probability of choosing this action, and would also cause a decay in the probability of choosing other better actions (since the probabilities are normalized to $1$).

To further showcase the ability of RPO to address overoptimization (through the lense of probability), consider the following theoretical example with only one state and three actions \citep{xu2024dpo} where we can track everything clearly. It has three actions ${a, b, c}$ with $R^\star(a) = 1, R^\star(b)=0.5,R^\star(c)=0$. The reference policy $\pi^{\mathrm{ref}}(a)=\pi^{\mathrm{ref}}(b)=0.4, \pi^{\mathrm{ref}}(c)=0.1$, and the dataset consists of one data point $\mathcal{D} = (a,b,1)$ (meaning action $a$ is preferred in the data). Then an ideally solved DPO objective would be $\pi_{\mathrm{DPO}}$ as long as $\pi^{\mathrm{DPO}}(b)=0$, and the value of $\pi^{\mathrm{DPO}}(a)$ can be arbitrarily chosen in $[0,1]$. Thus a possible solution to DPO would be $\pi^{\mathrm{DPO}}(a)=0.5,\pi^{\mathrm{DPO}}(b)=0$, and by the normalizing condition $\pi^{\mathrm{DPO}}(c)=0.5$, which is undesirable since the action $c$ has reward $R^{\star}(c)=0$. In contrast, solving the RPO objective would additionally require the maximization of $\pi_{\mathrm{RPO}}(a)$ due to the SFT regularization term, and thus the solution is shifted towards $\pi_{\mathrm{RPO}}(a)=1, \pi_{\mathrm{RPO}}(b)=\pi_{\mathrm{RPO}}(c)=0$, which is better than the DPO policy. Thus, RPO is able to prevent overoptimization towards poor actions that are less covered by the dataset (action $c$ here), therefore resulting in a better policy.

\paragraph{About the relationships and distinctions between PTX loss in \citep{taori2023stanford} and the SFT loss of RPO.}

The original PTX loss is an imitation loss calculated on the pretraining data. 
In contrast, the SFT loss in the RPO objective is an imitation loss calculated on the RLHF dataset. 
In more specific, our experiments use this SFT loss to imitate the chosen responses in the RLHF dataset. 
Thus the relationship is that they are both imitation loss which aims to mimic certain data distribution. 
The distinction is that they are calculated on different data sources. 
The SFT loss in the RPO objective naturally comes from our theoretical algorithm and provably serves as an important regularization term to mitigate overoptimization in offline RLHF.

\paragraph{About the computational complexity of the SFT loss gradient.} 
According to the paragraph \textbf{Practical implementation} in Section~\ref{sec: experiment}, RPO adds an additional SFT loss (the log probability of the chosen labels in the preference dataset) on the original DPO loss, where we highlight that the SFT loss is actually an intermediate quantity in the calculation of the DPO loss. 
Therefore, our proposed method does not incur any additional computation overhead compared with the vanilla DPO.

\section{Proofs for Sample Complexity Analysis}\label{sec: proof sample complexity}
\subsection{Concepts}
We provide some useful concepts for the simplicity of later discussions. 
\begin{itemize}
    \item Hellinger distance $D_{\mathrm{Hellinger}}(p\|q)$ between two probability density functions $p$ and $q$ defined on $\mathcal
X$ is defined as 
\begin{align*} D_{\mathrm{Hellinger}}(p\|q) = \frac{1}{2}{\int_{x\in\mathcal{X}}\left(\sqrt{p(x)} - \sqrt{q(x)}\right)^2}{}\mathrm{d}x.
\end{align*}
   \item Total variation (TV) distance $D_{\mathrm{TV}}(p\|q)$ between two probability density functions $p$ and $q$ defined on $\mathcal
X$ is defined as 
\begin{align*} D_{\mathrm{TV}}(p\|q) = \frac{1}{2}{\int_{x\in\mathcal{X}}|p(x) - q(x)}|\mathrm{d}x.
\end{align*}
 \item Kullback–Leibler (KL) divergence $\text{KL}(p\|q)$ between two probability density functions $p$ and $q$ defined on $\mathcal
X$ is defined as 
\begin{align*} \text{KL}(p\|q) = \int_{x\in\mathcal{X}}\log\left(\frac{ p(x)}{q(x)}\right)p(x)\mathrm{d}x.
\end{align*}
\item We denote $\mathcal{N}_\epsilon(\mathcal{F},\|\cdot\|_\infty)$ as the $\epsilon$-covering number \citep{zhou2002covering} for function class $\mathcal{F}$ under the infinity norm $\|\cdot\|_\infty$. 
Widely used in the theoretical analysis, the $\epsilon$-covering number characterizes the complexity of the function class $\mathcal{F}$. 
\end{itemize}

\subsection{Proof of Theorem~\ref{thm:main}}\label{subsec: proof main}

\begin{proof}[Proof of Theorem~\ref{thm:main}]\label{subsec: proof thm main}
    By definition, the suboptimality gap of $\widehat{\pi}$ w.r.t. $\pi$ is decomposed as following, 
    \begin{align}
        &\mathrm{Gap}^\pi(\widehat{\pi})\\
        &\qquad=\mathbb{E}_{x\sim d_0,a\sim \pi(\cdot|x)}\big[r^{\star}(x,a)\big] - \mathbb{E}_{x\sim d_0,a\sim \widehat{\pi}(\cdot|x)}\big[r^{\star}(x,a)\big] \\
        &\qquad= \mathbb{E}_{x\sim d_0,a^1\sim \pi(\cdot|x), a^0\sim \pi^{\mathrm{base}}(\cdot|x)}\Big[r^{\star}(x,a^1) - r^{\star}(x,a^0) - \beta\cdot \mathrm{KL}\big(\pi(\cdot|x)\|\pi^{\mathrm{ref}}(\cdot|x)\big)\Big] \\
        &\qquad\qquad - \eta^{-1}\cdot\min_{r\in\mathcal{R}}\left\{\eta\cdot\mathbb{E}_{\substack{x\sim d_0, a^1\sim\widehat{\pi}(\cdot|x),\\ a^0\sim \pi^{\mathrm{base}}(\cdot|x)}}\Big[r(x,a^1) - r(x,a^0) -\beta\cdot \mathrm{KL}\big(\widehat{\pi}(\cdot|x)\|\pi^{\mathrm{ref}}(\cdot|x)\big)\Big] +  {\cL}_{\cD}(r)\right\}\\
        &\qquad\qquad +\eta^{-1}\cdot\min_{r\in\mathcal{R}}\left\{\eta\cdot\mathbb{E}_{\substack{x\sim d_0, a^1\sim\widehat{\pi}(\cdot|x),\\ a^0\sim \pi^{\mathrm{base}}(\cdot|x)}}\Big[r(x,a^1) - r(x,a^0) -\beta\cdot \mathrm{KL}\big(\widehat{\pi}(\cdot|x)\|\pi^{\mathrm{ref}}(\cdot|x)\big)\Big] +  {\cL}_{\cD}(r)\right\}\\
        &\qquad\qquad - \mathbb{E}_{x\sim d_0,a^1\sim \widehat{\pi}(\cdot|x), a^0\sim \pi^{\mathrm{base}}(\cdot|x)}\Big[r^{\star}(x,a^1) - r^{\star}(x,a^0)- \beta\cdot \mathrm{KL}\big(\widehat{\pi}(\cdot|x)\|\pi^{\mathrm{ref}}(\cdot|x)\big)\Big] \\
        &\qquad\qquad + \beta \cdot \mathbb{E}_{x\sim d_0}\Big[\mathrm{KL}\big(\pi(\cdot|x)\|\pi^{\mathrm{ref}}(\cdot|x)\big) - \mathrm{KL}\big(\widehat{\pi}(\cdot|x)\|\pi^{\mathrm{ref}}(\cdot|x)\big) \Big]\\
        &\qquad:=\text{Term (A)} + \text{Term (B)} + \text{Term (C)} \label{eq: gap proof},
    \end{align}
    where in the above Term (A), Term (B), and Term (C) are abbreviations for 
    \begin{align}
        \text{Term (A)} 
        &= \mathbb{E}_{x\sim d_0,a^1\sim \pi(\cdot|x),a^0\sim \pi^{\mathrm{base}}(\cdot|x)}\Big[r^{\star}(x,a^1) -r^{\star}(x,a^0) - \beta\cdot \mathrm{KL}\big(\pi(\cdot|x)\|\pi^{\mathrm{ref}}(\cdot|x)\big)\Big] \\
        &\qquad -\eta^{-1}\cdot\min_{r\in\mathcal{R}}\left\{\eta\cdot\mathbb{E}_{\substack{x\sim d_0, a^1\sim\widehat{\pi}(\cdot|x),\\ a^0\sim \pi^{\mathrm{base}}(\cdot|x)}}\Big[r(x,a^1) - r(x,a^0) -\beta\cdot \mathrm{KL}\big(\widehat{\pi}(\cdot|x)\|\pi^{\mathrm{ref}}(\cdot|x)\big)\Big] + {\cL}_{\cD}(r)\right\},
    \end{align}
    and 
    \begin{align}
        \text{Term (B)} &= \eta^{-1}\cdot\min_{r\in\mathcal{R}}\left\{\eta\cdot\mathbb{E}_{\substack{x\sim d_0, a^1\sim\widehat{\pi}(\cdot|x),\\ a^0\sim \pi^{\mathrm{base}}(\cdot|x)}}\Big[r(x,a^1) - r(x,a^0) -\beta\cdot \mathrm{KL}\big(\widehat{\pi}(\cdot|x)\|\pi^{\mathrm{ref}}(\cdot|x)\big)\Big] +  {\cL}_{\cD}(r)\right\}\\
        &\qquad   - \mathbb{E}_{x\sim d_0,a^1\sim \widehat{\pi}(\cdot|x),a^0\sim \pi^{\mathrm{base}(\cdot|x)}}\Big[r^{\star}(x,a^1) -r^{\star}(x,a^0) - \beta\cdot \mathrm{KL}\big(\widehat{\pi}(\cdot|x)\|\pi^{\mathrm{ref}}(\cdot|x)\big)\Big],
    \end{align}
    and
    \begin{align}
        \text{Term (C)} = \beta \cdot \mathbb{E}_{x\sim d_0}\Big[\mathrm{KL}\big(\pi(\cdot|x)\|\pi^{\mathrm{ref}}(\cdot|x)\big) - \mathrm{KL}\big(\widehat{\pi}(\cdot|x)\|\pi^{\mathrm{ref}}(\cdot|x)\big) \Big].
    \end{align}
    In the following, we analyze Term (A) and Term (B) respectively.

    \paragraph{Upper bound Term (A).} 
    Notice that by the optimality of our choice of policy $\widehat{\pi}$ in \eqref{eq: target max min}, we have 
    \begin{align}
        &\text{Term (A)} \\
        &\qquad = \mathbb{E}_{x\sim d_0,a^1\sim \pi(\cdot|x),a^0\sim \pi^{\mathrm{base}}(\cdot|x)}\Big[r^{\star}(x,a^1) -r^{\star}(x,a^0) - \beta\cdot \mathrm{KL}\big(\pi(\cdot|x)\|\pi^{\mathrm{ref}}(\cdot|x)\big)\Big] \label{eq: term a}\\
        &\qquad\qquad  -\eta^{-1}\cdot\min_{r\in\mathcal{R}}\left\{ \eta\cdot\mathbb{E}_{\substack{x\sim d_0, a^1\sim\widehat{\pi}(\cdot|x),\\ a^0\sim \pi^{\mathrm{base}}(\cdot|x)}}\Big[r(x,a^1) - r(x,a^0) -\beta\cdot \mathrm{KL}\big(\widehat{\pi}(\cdot|x)\|\pi^{\mathrm{ref}}(\cdot|x)\big)\Big] + {\cL}_{\cD}(r)\right\}\\
        &\qquad\leq  \mathbb{E}_{x\sim d_0,a^1\sim \pi(\cdot|x),a^0\sim \pi^{\mathrm{ref}}(\cdot|x)}\Big[r^{\star}(x,a^1) -r^{\star}(x,a^0) - \beta\cdot \mathrm{KL}\big(\pi(\cdot|x)\|\pi^{\mathrm{ref}}(\cdot|x)\big)\Big] \\
        &\qquad\qquad  -\eta^{-1}\cdot \min_{r\in\mathcal{R}}\left\{\eta\cdot\mathbb{E}_{\substack{x\sim d_0, a^1\sim\pi(\cdot|x),\\ a^0\sim \pi^{\mathrm{base}}(\cdot|x)}}\Big[r(x,a^1) - r(x,a^0) -\beta\cdot \mathrm{KL}\big(\pi(\cdot|x)\|\pi^{\mathrm{ref}}(\cdot|x)\big)\Big] +  {\cL}_{\cD}(r)\right\}\\
        &\qquad= \max_{r\in\mathcal{R}}\Bigg\{\mathbb{E}_{x\sim d_0,a^1\sim \pi(\cdot|x),a^0\sim \pi^{\mathrm{base}}(\cdot|x)}\Big[\big(r^{\star}(x,a^1) -r^{\star}(x,a^0)\big) - \big(r(x,a^1) -r(x,a^0)\big)\Big] - \eta^{-1}\cdot \cL_{\cD}(r)\Bigg\},  
    \end{align}
    where in the inequality we apply the optimality of the choice of policy $\widehat{\pi}$ in \eqref{eq: target max min}.

    \paragraph{Upper bound Term (B).} For this term, we directly consider the following bound, 
    \begin{align}
        &\text{Term (B)}\\
        &\qquad = \eta^{-1}\cdot\min_{r\in\mathcal{R}}\left\{\eta\cdot\mathbb{E}_{\substack{x\sim d_0, a^1\sim\widehat{\pi}(\cdot|x),\\ a^0\sim \pi^{\mathrm{base}}(\cdot|x)}}\Big[r(x,a^1) - r(x,a^0) -\beta\cdot \mathrm{KL}\big(\widehat{\pi}(\cdot|x)\|\pi^{\mathrm{ref}}(\cdot|x)\big)\Big] +  {\cL}_{\cD}(r)\right\}\\
        &\qquad \qquad  - \mathbb{E}_{x\sim d_0,a^1\sim \widehat{\pi}(\cdot|x),a^0\sim \pi^{\mathrm{base}(\cdot|x)}}\Big[r^{\star}(x,a^1) -r^{\star}(x,a^0) - \beta\cdot \mathrm{KL}\big(\widehat{\pi}(\cdot|x)\|\pi^{\mathrm{ref}}(\cdot|x)\big)\Big]\\
        &\qquad  \leq \mathbb{E}_{x\sim d_0, a^1\sim\widehat{\pi}(\cdot|x),a^0\sim \pi^{\mathrm{base}}(\cdot|x)}\Big[r^{\star}(x,a^1) - r^{\star}(x,a^0) -\beta\cdot \mathrm{KL}\big(\widehat{\pi}(\cdot|x)\|\pi^{\mathrm{ref}}(\cdot|x)\big)\Big] + \eta^{-1}\cdot {\cL}_{\cD}(r^{\star})\\
        &\qquad \qquad  - \mathbb{E}_{x\sim d_0,a^1\sim \widehat{\pi}(\cdot|x),a^0\sim \pi^{\mathrm{base}(\cdot|x)}}\Big[r^{\star}(x,a^1) -r^{\star}(x,a^0) - \beta\cdot \mathrm{KL}\big(\widehat{\pi}(\cdot|x)\|\pi^{\mathrm{ref}}(\cdot|x)\big)\Big] \\
        &\qquad= \eta^{-1}\cdot {\cL}_{\cD}(r^{\star}),\label{eq: term b}
    \end{align}
    where in the inequality we apply the fact that $r^{\star}\in\cR$ by Assumption~\ref{as: regularity}.

    \paragraph{Combining Term (A), Term (B), and Term (C).}
    Now by \eqref{eq: gap proof}, \eqref{eq: term a}, and \eqref{eq: term b}, we have that 
    \allowdisplaybreaks
    \begin{align}
        &\mathrm{Gap}_{\beta}^\pi(\widehat{\pi}) = \text{Term (A)} + \text{Term (B)} + \text{Term (C)}\label{eq: proof combine}\\
        &\qquad \leq \max_{r\in\mathcal{R}}\left\{\mathbb{E}_{\substack{x\sim d_0,a^1\sim \pi(\cdot|x),\\a^0\sim \pi^{\mathrm{base}}(\cdot|x)}}\Big[\big(r^{\star}(x,a^1) -r^{\star}(x,a^0)\big) - \big(r(x,a^1) -r(x,a^0)\big)\Big] +\eta^{-1}\cdot\Big(\cL_{\cD}(r^{\star}) -  \cL_{\cD}(r)\Big)\right\}\\
        &\qquad\qquad + \beta \cdot \mathbb{E}_{x\sim d_0}\Big[\mathrm{KL}\big(\pi(\cdot|x)\|\pi^{\mathrm{ref}}(\cdot|x)\big) - \mathrm{KL}\big(\widehat{\pi}(\cdot|x)\|\pi^{\mathrm{ref}}(\cdot|x)\big) \Big].
    \end{align}
    In the following, we upper bound the right hand side of \eqref{eq: proof combine} via relating the MLE loss difference term to the reward difference term through a careful analysis of the preference model. 
    On the one hand, we invoke Lemma~\ref{lem: concentration} to give an upper bound of the difference of the MLE loss as following, with probability at least $1-\delta$ over random samples and $\varepsilon = (6\cdot(1+e^R)\cdot N)^{-1}$, for any reward model $r\in\mathcal{R}$, it holds that
    \begin{align}
        \cL_{\cD}(r^{\star}) -  \cL_{\cD}(r) 
        & \leq -2\cdot \mathbb{E}_{(x,a^1,a^0)\sim \mu_{\cD}(\cdot,\cdot,\cdot)}\Big[D_{\mathrm{Hellinger}}^2\big(\mathbb{P}_{r^{\star}}(\cdot|x,a^1,a^0)\|\mathbb{P}_{r}(\cdot|x,a^1,a^0)\big)\Big] \\
        &\qquad+ \frac{3}{N}\cdot\log\left(\frac{\cN_{\varepsilon}(\cR,\|\cdot\|_{\infty})}{\delta}\right),
    \end{align}
    where we recall that we use the subscript $r$ in $\mathbb{P}_r$ to emphasize the dependence of the probabilistic model on the reward model. 
    Here $\cN_{\varepsilon}(\cR,\|\cdot\|_{\infty})$ denotes the $\varepsilon$-covering number of the reward model class and $R$ is the upper bound on the reward functionss (Assumption~\ref{as: regularity}).
    Now to facilitate the calculation, we lower bound the Hellinger distance by total variation (TV) distance as
    \begin{align}
        D_{\mathrm{Hellinger}}^2\big(\mathbb{P}_{r^{\star}}(\cdot|x,a^1,a^0)\|\mathbb{P}_{r}(\cdot|x,a^1,a^0)\big) \geq D_{\mathrm{TV}}^2\big(\mathbb{P}_{r^{\star}}(\cdot|x,a^1,a^0)\|\mathbb{P}_{r}(\cdot|x,a^1,a^0)\big),
    \end{align}
    By the expression of the probability model $\mathbb{P}_r$, we can further write the TV distance above as 
    \begin{align}
        &D_{\mathrm{TV}}\big(\mathbb{P}_{r^{\star}}(\cdot|x,a^1,a^0)\|\mathbb{P}_{r}(\cdot|x,a^1,a^0)\big)\\ &\qquad= \frac{1}{2}\cdot \Big|\sigma\big(r^{\star}(x,a^1) - r^{\star}(x,a^0)\big) - \sigma\big(r(x,a^1) - r(x,a^0)\big)\Big|\\
        &\qquad\qquad + \frac{1}{2}\cdot \Big|\sigma\big(r^{\star}(x,a^0) - r^{\star}(x,a^1)\big) - \sigma\big(r(x,a^0) - r(x,a^1)\big)\Big|\\
        &\qquad = \Big|\sigma\big(r^{\star}(x,a^1) - r^{\star}(x,a^0)\big) - \sigma\big(r(x,a^1) - r(x,a^0)\big)\Big|,\label{eq: tv calculation}
    \end{align}
    where in the second equality we use the fact that $\sigma(-z) = 1-\sigma(z)$. 
    Now by Lemma~\ref{lem: sigmoid} and the condition that $r(x,a)\in[0,R]$ for any $(x,a,r)\in\cX\times\cA\times\mathcal{R}$ (Assumption~\ref{as: regularity}), we know that 
    \begin{align}
        &\Big|\sigma\big(r^{\star}(x,a^1) - r^{\star}(x,a^0)\big) - \sigma\big(r(x,a^1) - r(x,a^0)\big)\Big|  \geq \kappa \cdot \Big|\big(r^{\star}(x,a^1) - r^{\star}(x,a^0)\big) - \big(r(x,a^1) - r(x,a^0)\big)\Big|,
    \end{align}
    where $\kappa = 1/(1+\exp( R))^2$.
    As a result, the difference of the MLE loss is upper bounded by 
    \begin{align}
        \cL_{\cD}(r^{\star}) -  \cL_{\cD}(r)
        & \leq -2\kappa^2 \cdot \mathbb{E}_{(x,a^1,a^0)\sim \mu_{\cD}(\cdot,\cdot,\cdot)}\bigg[\Big|\big(r^{\star}(x,a^1) - r^{\star}(x,a^0)\big) - \big(r(x,a^1) - r(x,a^0)\big)\Big|^2\bigg] \\
        &\qquad  + \frac{3}{N}\cdot\log\left(\frac{\cN_{\varepsilon}(\cR,\|\cdot\|_{\infty})}{\delta}\right).\label{eq: mle difference}
    \end{align}
    On the other hand, the reward difference term in \eqref{eq: proof combine}, which is evaluated on actions from $\pi$ and $\pi^{\mathrm{base}}$, can be related to the reward difference evaluated on the data distribution $\mu_{\cD}$ via Assumption~\ref{as: coverage}, i.e., 
    \allowdisplaybreaks
    \begin{align}
        &\mathbb{E}_{x\sim d_0,a^1\sim \pi(\cdot|x),a^0\sim \pi^{\mathrm{base}}(\cdot|x)}\Big[\big(r^{\star}(x,a^1) -r^{\star}(x,a^0)\big) - \big(r(x,a^1) -r(x,a^0)\big)\Big]\label{eq: reward difference}\\
        &\quad \leq C_{\mu_{\cD}}(\cR; \pi, \pi^{\mathrm{base}})  \sqrt{\mathbb{E}_{(x,a^1,a^0)\sim \mu_{\cD}}\left[\Big|\big(r^{\star}(x,a^1) -r^{\star}(x,a^0)\big) - \big(r(x,a^1) -r(x,a^0)\big)\Big|^2\right]}.
    \end{align}
    Finally, combining \eqref{eq: mle difference}, \eqref{eq: reward difference}, and \eqref{eq: proof combine}, denoting 
    \begin{align}
        \Delta_{r}:= \sqrt{\mathbb{E}_{(x,a^1,a^0)\sim \mu_{\cD}}\left[\Big|\big(r^{\star}(x,a^1) -r^{\star}(x,a^0)\big) - \big(r(x,a^1) -r(x,a^0)\big)\Big|^2\right]},
    \end{align}
    we have that 
    \begin{align}
        \mathrm{Gap}^\pi(\widehat{\pi}) &\leq \max_{r\in\mathcal{R}}\Big\{C_{\mu_{\cD}}(\cR; \pi, \pi^{\mathrm{base}})\cdot \Delta_{r} - 2\eta^{-1}\kappa^2\cdot \Delta_{r}^2 \Big\}+  \frac{3}{\eta N}\cdot\log\left(\frac{\cN_{\varepsilon}(\cR,\|\cdot\|_{\infty})}{\delta}\right)\\
        &\qquad + \beta \cdot \mathbb{E}_{x\sim d_0}\Big[\mathrm{KL}\big(\pi(\cdot|x)\|\pi^{\mathrm{ref}}(\cdot|x)\big) - \mathrm{KL}\big(\widehat{\pi}(\cdot|x)\|\pi^{\mathrm{ref}}(\cdot|x)\big) \Big]\\
        &\leq \frac{\big(C_{\mu_{\cD}}(\cR; \pi, \pi^{\mathrm{base}})\big)^2\eta}{8\kappa^2 } + \frac{3}{\eta N}\cdot \log\left(\frac{\cN_{\varepsilon}(\cR,\|\cdot\|_{\infty})}{\delta}\right) + \beta \cdot \mathbb{E}_{x\sim d_0}\Big[\mathrm{KL}\big(\pi(\cdot|x)\|\pi^{\mathrm{ref}}(\cdot|x)\big) \Big],
    \end{align}
    where in the second inequality we use that fact that $az-bz^2\leq a^2/(4b)$ for any $z\in\mathbb{R}$ and that KL-divergence is non-negative.
    Consequently, with the choice of 
    \begin{align}
        \eta = 2\sqrt{6}\cdot\sqrt{\frac{\log\left(\cN_{\varepsilon}(\cR,\|\cdot\|_{\infty})/\delta\right)}{N}},\quad \beta = \frac{1}{\sqrt{N}},\quad \kappa = \frac{1}{(1+\exp(R))^2},
    \end{align}
    we conclude that with probability at least $1-\delta$ and $\varepsilon = (6\cdot(1+e^R)\cdot N)^{-1}$, 
    \begin{align}
         &\mathrm{Gap}^\pi(\widehat{\pi}) \leq \frac{\sqrt{6}\big(1+\exp(R)\big)^2 \left(\big(C_{\mu_{\cD}}(\cR; \pi, \pi^{\mathrm{base}})\big)^2+1\right) \iota + 4 \mathbb{E}_{x\sim d_0}\Big[\mathrm{KL}\big(\pi(\cdot|x)\|\pi^{\mathrm{ref}}(\cdot|x)\big) \Big]}{4\sqrt{N}},
    \end{align}
    where we denote $\iota =  \sqrt{\log\left(\cN_{\varepsilon}(\cR,\|\cdot\|_{\infty})/\delta\right)}$ with $\varepsilon = (6\cdot(1+e^R)\cdot N)^{-1}$.
    This proves Theorem \ref{thm:main}.
\end{proof}

\subsection{Technical Lemmas}

\begin{lemma}[Uniform concentration]\label{lem: concentration}
    Consider the MLE loss \eqref{eq: L D} and define the approximation error as $\varepsilon = (6\cdot(1+e^R)\cdot N)^{-1}$ where $R$ is the upper bound on the reward functions (Assumption~\ref{as: coverage}).
    Suppose that the reward model class $\cR$ has a finite $\varepsilon$-covering number $\cN_{\varepsilon}(\cR,\|\cdot\|_{\infty})<\infty$.
    Then for any $\delta<1/e$ it holds with probability at least $1-\delta$ that 
\begin{align}
    \cL_{\cD}(r^{\star}) -  \cL_{\cD}(r) 
        &\leq -2\cdot \mathbb{E}_{(x,a^1,a^0)\sim \mu_{\cD}(\cdot,\cdot,\cdot)}\Big[D_{\mathrm{Hellinger}}^2\big(\mathbb{P}_{r^{\star}}(\cdot|x,a^1,a^0)\|\mathbb{P}_{r}(\cdot|x,a^1,a^0)\big)\Big]  \\
        & \qquad + \frac{3}{N}\cdot\log\left(\frac{\cN_{\varepsilon}(\cR,\|\cdot\|_{\infty})}{\delta}\right).
\end{align}
\end{lemma}

\begin{proof}[Proof of Lemma~\ref{lem: concentration}]
    For notational simplicity, we use $\cC_{\varepsilon}(\cR,\|\cdot\|_{\infty})$ to denote an $\varepsilon$-cover of the reward model class $\cR$ under the $\|\cdot\|_{\infty}$-norm.
    It holds that $\cN_{\varepsilon}(\cR,\|\cdot\|_{\infty}) = |\cC_{\varepsilon}(\cR,\|\cdot\|_{\infty})|$.
    First we invoke Proposition 5.3 of \cite{liu2024maximize} to obtain a uniform concentration over the finite set of  $\varepsilon$-cover $\cC_{\varepsilon}(\cR,\|\cdot\|_{\infty})$. 
    Specifically, with probability at least $1-\delta$, for any $r\in \cC_{\varepsilon}(\cR,\|\cdot\|_{\infty})$, 
    \begin{align}
        \cL_{\cD}(r^{\star}) -  \cL_{\cD}(r) 
        &\leq -2\cdot \mathbb{E}_{(x,a^1,a^0)\sim \mu_{\cD}(\cdot,\cdot,\cdot)}\Big[D_{\mathrm{Hellinger}}^2\big(\mathbb{P}_{r^{\star}}(\cdot|x,a^1,a^0)\|\mathbb{P}_{r}(\cdot|x,a^1,a^0)\big)\Big] \\
        &\qquad+ \frac{2}{N}\cdot\log\left(\frac{\cN_{\varepsilon}(\cR,\|\cdot\|_{\infty})}{\delta}\right).\label{eq: finite uniform}
    \end{align}
    Now for any reward model $r\in\cR$, we take a $r^{\dagger}\in\cC_{\varepsilon}(\cR,\|\cdot\|_{\infty})$ satisfying $\|r - r^{\dagger}\|_{\infty}\leq \varepsilon$.
    We have
    \allowdisplaybreaks
    \begin{align}
        \cL_{\cD}(r^{\star}) -  \cL_{\cD}(r) 
        & = \cL_{\cD}(r^{\star}) -  \cL_{\cD}(r^{\dagger}) +  \cL_{\cD}(r^{\dagger}) -  \cL_{\cD}(r)\\
        & \leq -2\cdot \mathbb{E}_{(x,a^1,a^0)\sim \mu_{\cD}(\cdot,\cdot,\cdot)}\Big[D_{\mathrm{Hellinger}}^2\big(\mathbb{P}_{r^{\star}}(\cdot|x,a^1,a^0)\|\mathbb{P}_{r^\dagger}(\cdot|x,a^1,a^0)\big)\Big] \\
        &\qquad  + \frac{2}{N}\cdot\log\left(\frac{\cN_{\varepsilon}(\cR,\|\cdot\|_{\infty})}{\delta}\right) +  \cL_{\cD}(r^{\dagger}) -  \cL_{\cD}(r)\\
        & \leq -2\cdot \mathbb{E}_{(x,a^1,a^0)\sim \mu_{\cD}(\cdot,\cdot,\cdot)}\Big[D_{\mathrm{Hellinger}}^2\big(\mathbb{P}_{r^{\star}}(\cdot|x,a^1,a^0)\|\mathbb{P}_{r}(\cdot|x,a^1,a^0)\big)\Big] \\
        &\qquad+ \frac{2}{N}\cdot\log\left(\frac{\cN_{\varepsilon}(\cR,\|\cdot\|_{\infty})}{\delta}\right) +  \cL_{\cD}(r^{\dagger}) -  \cL_{\cD}(r)\\
        &\qquad\qquad + 4 \cdot \mathbb{E}_{(x,a^1,a^0)\sim \mu_{\cD}(\cdot,\cdot,\cdot)}\Big[D_{\mathrm{Hellinger}}^2\big(\mathbb{P}_{r^{\dagger}}(\cdot|x,a^1,a^0)\|\mathbb{P}_{r}(\cdot|x,a^1,a^0)\big)\Big],\label{eq: uniform concentration 1}
    \end{align}
    where in the fir inequality we use \eqref{eq: finite uniform} for $r^{\dagger}$ and in the second inequality we utilize the triangular inequality for Hellinger distance.
    Therefore, it remains to upper bound the approximation error induced by $r^{\dagger}$.
    On the one hand, by the definition of $\cL_\cD$ in \eqref{eq: L D}, we have that 
    \begin{align}
        \cL_{\cD}(r^{\dagger}) -  \cL_{\cD}(r) 
        & = \frac{1}{N}\sum_{i=1}^N y_i\cdot \log \left(\frac{\sigma\big(r(x_i,a^1_i) - r(x_i,a^0_i)\big)}{\sigma\big(r^\dagger(x_i,a^1_i) - r^\dagger(x_i,a^0_i)\big)}\right)\\
        &\qquad +\frac{1}{N}\sum_{i=1}^N(1-y_i)\cdot \log \left(\frac{\sigma\big(r(x_i,a^0_i) - r(x_i,a^1_i)\big)}{\sigma\big(r^\dagger(x_i,a^0_i) - r^\dagger(x_i,a^1_i)\big)}\right).
    \end{align}
    Use the inequality that $\log(x) \leq x-1$, we can further upper bound $\cL_{\cD}(r^{\dagger}) -  \cL_{\cD}(r)$ by 
    \begin{align}
        \cL_{\cD}(r^{\dagger}) -  \cL_{\cD}(r) 
        &\leq  \frac{1}{N}\sum_{i=1}^N y_i\cdot \frac{\sigma\big(r(x_i,a^1_i) - r(x_i,a^0_i)\big) - \sigma\big(r^\dagger(x_i,a^1_i) - r^\dagger(x_i,a^0_i)\big)}{\sigma\big(r^\dagger(x_i,a^1_i) - r^\dagger(x_i,a^0_i)\big)} \\
        &\qquad \qquad +\frac{1}{N}\sum_{i=1}^N(1-y_i)\cdot\frac{\sigma\big(r(x_i,a^0_i) - r(x_i,a^1_i)\big) - \sigma\big(r^\dagger(x_i,a^0_i) - r^\dagger(x_i,a^1_i)\big)}{\sigma\big(r^\dagger(x_i,a^0_i) - r^\dagger(x_i,a^1_i)\big)}.
    \end{align}
    Now since $\|r^{\dagger} - r\|_{\infty}\leq \varepsilon$ and $r^{\dagger}\in[0,R]$, invoking Lemma~\ref{lem: sigmoid}, we can derive that 
    \begin{align}
        \cL_{\cD}(r^{\dagger}) -  \cL_{\cD}(r) &\leq \frac{1}{N}\sum_{i=1}^N \frac{\big|\big(r(x_i,a^1_i) - r(x_i,a^0_i)\big) - \big(r^\dagger(x_i,a^1_i) - r^\dagger(x_i,a^0_i)\big)\big|}{(1+e^{R})^{-1}} \\
        &\qquad + \frac{1}{N}\sum_{i=1}^N\frac{\big|\big(r(x_i,a^0_i) - r(x_i,a^1_i)\big) - \big(r^\dagger(x_i,a^0_i) - r^\dagger(x_i,a^1_i)\big)\big|}{(1+e^{R})^{-1}}\\
        & \leq 4\cdot \|r^{\dagger} - r\|_{\infty}\cdot (1+e^R)\leq 4\varepsilon\cdot (1+e^R).\label{eq: approx error 1}
    \end{align}
    On the other hand, we upper bound the hellinger distance between $\mathbb{P}_r$ and $\mathbb{P}_{r^{\dagger}}$, for any $(x, a^1, a^0)\in\cX\times\cA\times\cA$,
    \begin{align}
        &D_{\mathrm{Hellinger}}^2\big(\mathbb{P}_{r^{\dagger}}(\cdot|x,a^1,a^0)\|\mathbb{P}_{r}(\cdot|x,a^1,a^0)\big)  \\
        &\qquad \leq D_{\mathrm{TV}}\big(\mathbb{P}_{r^{\dagger}}(\cdot|x,a^1,a^0)\|\mathbb{P}_{r}(\cdot|x,a^1,a^0)\big) \\
        &\qquad = \Big|\sigma\big(r^{\dagger}(x,a^1) - r^{\dagger}(x,a^0)\big) - \sigma\big(r(x,a^1) - r(x,a^0)\big)\Big| \\
        & \qquad \leq \Big|\big(r^{\dagger}(x,a^1) - r^{\dagger}(x,a^0)\big) - \big(r(x,a^1) - r(x,a^0)\big)\Big| \\
        &\qquad \leq 2\cdot \|r^{\dagger} - r\|_{\infty}\leq 2\varepsilon,\label{eq: approx error 2}
    \end{align}
    where the first inequality uses the fact that $D_{\mathrm{Hellinger}}^2 \leq D_{\mathrm{TV}}$, the equality uses the same argument as \eqref{eq: tv calculation}, and the second inequality applies Lemma~\ref{lem: sigmoid}.
    Finally, combining \eqref{eq: uniform concentration 1}, \eqref{eq: approx error 1}, and \eqref{eq: approx error 2}, we conclude that 
    \begin{align}
        \cL_{\cD}(r^{\star}) -  \cL_{\cD}(r) &\leq -2\cdot \mathbb{E}_{(x,a^1,a^0)\sim \mu_{\cD}(\cdot,\cdot,\cdot)}\Big[D_{\mathrm{Hellinger}}^2\big(\mathbb{P}_{r^{\star}}(\cdot|x,a^1,a^0)\|\mathbb{P}_{r}(\cdot|x,a^1,a^0)\big)\Big] \\
        &\qquad + \frac{2}{N}\cdot\log\left(\frac{\cN_{\varepsilon}(\cR,\|\cdot\|_{\infty})}{\delta}\right) + 6 \varepsilon\cdot(1+e^R). 
    \end{align}
    By taking the approximation error $\varepsilon = (6\cdot(1+e^R)\cdot N)^{-1}$, we conclude that for $\delta<e^{-1}$, with probability at least $1-\delta$, for any $r\in\cR$, it holds that 
    \begin{align}
        \cL_{\cD}(r^{\star}) -  \cL_{\cD}(r) 
        &\leq -2\cdot \mathbb{E}_{(x,a^1,a^0)\sim \mu_{\cD}(\cdot,\cdot,\cdot)}\Big[D_{\mathrm{Hellinger}}^2\big(\mathbb{P}_{r^{\star}}(\cdot|x,a^1,a^0)\|\mathbb{P}_{r}(\cdot|x,a^1,a^0)\big)\Big]  \\
        &\qquad\qquad  + \frac{3}{N}\cdot\log\left(\frac{\cN_{\varepsilon}(\cR,\|\cdot\|_{\infty})}{\delta}\right).
    \end{align}
    This completes the proof of Lemma~\ref{lem: concentration}.
\end{proof}

\begin{lemma}[Sigmoid function]\label{lem: sigmoid}
    For any real numbers $z_1,z_2\in[-R,R]$, it holds that 
    \begin{align}
        \kappa\cdot |z_1-z_2|\leq\left|\sigma(z_1) - \sigma(z_2)\right|\leq |z_1-z_2|,
    \end{align}
    where the constant $\kappa = 1/(1+\exp(R))^2$.
\end{lemma}

\begin{proof}[Proof of Lemma~\ref{lem: sigmoid}]
    Since the sigmoid function $\sigma(\cdot)$ is differentiable, we know that for any $z_1,z_2\in[-R,R]$, there exists some $\xi(z_1,z_2)\in[-R,R]$ such that 
    \begin{align}
        \sigma(z_1) - \sigma(z_2) = \sigma'\big(\xi(z_1,z_2)\big)\cdot(z_1 - z_2).
    \end{align}
    Notice that $\sigma'(z) = \sigma(z)\cdot(1-\sigma(z))$, we can obtain that
    \allowdisplaybreaks
    \begin{align}
        1\geq \sigma'\big(\xi(z_1,z_2)\big) &= \sigma\big(\xi(z_1,z_2)\big)\cdot\Big(1-\sigma\big(\xi(z_1,z_2)\big)\Big)\\
        &=\frac{1}{1+\exp(\xi(z_1,z_2))}\cdot\left(1-\frac{1}{1+\exp(\xi(z_1,z_2))}\right)\\
        &\geq \frac{1}{1+\exp(R)}\cdot\left(1-\frac{1}{1+\exp(-R)}\right)\\
        &= \frac{1}{(1+\exp(R))^2}.
    \end{align}
    This completes the proof of Lemma~\ref{lem: sigmoid}.
\end{proof}

\section{Proofs for Equivalence between Maximin and Minimax Objectives}\label{sec: further detail}

\subsection{Proof of Theorem~\ref{thm: equivalence formal}}\label{subsec: proof equivalence formal}

\begin{proof}[Proof of Theorem~\ref{thm: equivalence formal}]
    Consider denoting an auxiliary policy $\widehat{\pi}$ as 
    \begin{align}
        \widehat{\pi} \in \argmax_{\pi\in\Pi}\,\min_{r\in\cR} \phi(\pi,r).\label{eq: proof nash 0+}
    \end{align}
    By the definition of $\widehat{r}$ and $\widehat{\pi}$, the duality gap of $(\widehat{r}, \widehat{\pi})$, defined as 
    \begin{align}
        \mathrm{Dual}(\widehat{r}, \widehat{\pi}) := \max_{\pi\in\Pi}\phi(\pi, \widehat{r}) - \min_{r\in\cR}\phi(\widehat{\pi},r)
    \end{align}
    is zero. 
    This is because the following deduction, 
    \begin{align}
        \mathrm{Dual}(\widehat{r}, \widehat{\pi}) &= \left(\max_{\pi\in\Pi}\phi(\pi, \widehat{r}) - \min_{r\in\cR}\,\max_{\pi\in\Pi} \phi(\pi,r) \right) \\
        &\qquad + \left(\max_{\pi\in\Pi}\,\min_{r\in\cR} \phi(\pi,r) 
 - \min_{r\in\cR}\phi(\widehat{\pi},r)\right)  \\
 &= 0,\label{eq: proof nash 1}
    \end{align}
    where in the first equality we apply Lemma~\ref{prop: equivalence} that the minimax objective and the maximin objective are equivalent, and the last equality applies the definition of $\widehat{r}$ and $\widehat{\pi}$ respectively. 
    Note that we can rewrite the duality gap as following
    \begin{align}
        \mathrm{Dual}(\widehat{r}, \widehat{\pi}) = \left(\max_{\pi\in\Pi}\phi(\pi, \widehat{r}) - \phi(\widehat{\pi},\widehat{r}) \right)+ \left(\phi(\widehat{\pi},\widehat{r})
 - \min_{r\in\cR}\phi(\widehat{\pi},r)\right).\label{eq: proof nash 2}
    \end{align}
    Combining \eqref{eq: proof nash 1} and \eqref{eq: proof nash 2}, we can conclude that 
    \begin{align}
        \max_{\pi\in\Pi}\phi(\pi, \widehat{r}) = \phi(\widehat{\pi},\widehat{r})\quad \Rightarrow\quad \widehat{\pi} \in \argmax_{\pi\in\Pi}\phi(\widehat{r},\pi).\label{eq: proof nash 3}
    \end{align}
    Now comparing what $\pi_{\widehat{r}}$ and $\widehat{\pi}$ satisfy in \eqref{eq: equivalent formal 1} and \eqref{eq: proof nash 3} respectively, invoking Lemma~\ref{lem: unique pi} that the maximizer of $\phi(\cdot,r)$ given any $r\in\cR$ is unique on the support of $d_0$, we can conclude that 
    \begin{align}
        \pi_{\widehat{r}}(\cdot|x) = \widehat{\pi}(\cdot|x),\quad \forall x\in\mathrm{Supp}(d_0).\label{eq: proof nash 4}
    \end{align}
    Therefore, by \eqref{eq: proof nash 0+} and \eqref{eq: proof nash 4}, and the fact that $\phi(\pi,r)$ depends on $\pi$ only through its value on the support of $d_0$, we can conclude that 
    \begin{align}
        \pi_{\widehat{r}} \in \argmax_{\pi\in\Pi}\,\min_{r\in\cR} \phi(\pi,r). 
    \end{align}
    This finishes the proof of Theorem~\ref{thm: equivalence formal}.
\end{proof}

\subsection{Auxiliary Lemmas}

\begin{lemma}[Equivalence of maximin and minimax objectives]\label{prop: equivalence}
    For the policy class $\Pi$ defined in \eqref{eq: policy class} and the reward model class $\cR$ satisfying Assumption~\ref{ass: regularity}, it holds that the maximin objective is equivalent to the minimax objective, i.e.,
    \begin{align}
        \max_{\pi\in\Pi}\,\min_{r\in \mathcal{R}}\phi(\pi,r) = \min_{r\in \mathcal{R}}\,\max_{\pi\in\Pi}\phi(\pi,r).
    \end{align}
\end{lemma}

\begin{proof}[Proof of Lemma~\ref{prop: equivalence}]
    The foundation of this result is a minimax theorem given by \cite{fan1953minimax} (Lemma~\ref{lem: minimax}).
    In our setting, the policy class $\Pi$ is a nonempty set, and the reward model class $\mathcal{R}$ is a nonempty compact Hausdorff space.
    Furthermore, by our choice of the policy class $\Pi$ in \eqref{eq: policy class}, $\Pi$ is a convex set. 
    Meanwhile, the function $\phi$ is a concave function of $\pi\in\Pi$ since the dependence on $\pi$ is linear terms plus a negative KL term (concave).
    Finally, by our assumption, the function $\phi$ is convex-like on the reward model class $\cR$ and is also continuous on $\mathcal{R}$.
    As a result, all the conditions of Lemma~\ref{lem: minimax} are satisfied and the minimax theorem holds in our problem setup, finishing the proof of Lemma~\ref{prop: equivalence}.  
\end{proof}

\begin{lemma}[Minimax theorem \citep{fan1953minimax}]\label{lem: minimax} 
    Let $\cX$ be a nonempty set (not necessarily topologized) and $\cY$ be a nonempty compact topological space. 
    Let $f:\cX\times\cY\mapsto\mathbb{R}$ be lower semicontinuous on $\cY$.
    Suppose that $f$ is concave-like on $\cX$ and convex-like on $\cY$, i.e., for any $x_1,x_2\in\cX$, $\alpha\in[0,1]$, there exists $x_3\in\cX$ such that 
    \begin{align}
        f(x_3,\cdot) \geq \alpha \cdot f(x_1,\cdot) + (1-\alpha)\cdot f(x_2,\cdot)\,\,\text{on $\cY$,}
    \end{align}
    and for any $y_1,y_2\in\cY$, $\beta\in[0,1]$, there exists $y_3\in\cY$ such that 
    \begin{align}
        f(\cdot, y_3) \leq \beta \cdot f(\cdot, y_1) + (1-\beta)\cdot f(\cdot, y_2)\,\,\text{on $\cY$.}
    \end{align}
    Then the following equation holds,
    \begin{align}
         \max_{x\in\cX}\,\min_{y\in\cY} f(x,y)=\min_{y\in\cY}\,\max_{x\in\cX} f(x,y).
    \end{align}
\end{lemma}

\begin{lemma}[Unique maximizer of $\phi$]\label{lem: unique pi}
    Consider the function $\phi$ defined as 
    \begin{align}
        \phi(\pi, r)&:= \eta\cdot\mathbb{E}_{x\sim d_0, a^1\sim\pi(\cdot|x), a^0\sim \pi^{\mathrm{base}}(\cdot|x)}\Big[r(x,a^1) - r(x,a^0) -\beta\cdot D_{\mathrm{KL}}\big(\pi(\cdot|x)\|\pi^{\mathrm{ref}}(\cdot|x)\big)\Big] \\
        &\qquad +  {\cL}_{\cD}(r). 
    \end{align}
    Then given any $r\in\cR$, the maximimzer of $\phi(\cdot,r)$ is unique on the support of $d_0$. 
    \begin{proof}[Proof of Lemma~\ref{lem: unique pi}]
        Given any $r\in\cR$, consider that 
        \begin{align}
            &\max_{\pi\in\Pi} \phi(\pi, r)\\
            &\qquad= \eta\cdot \max_{\pi\in\Pi}\left\{\mathbb{E}_{x\sim d_0, a^1\sim\pi(\cdot|x)}\Big[r(x,a^1)  -\beta\cdot D_{\mathrm{KL}}\big(\pi(\cdot|x)\|\pi^{\mathrm{ref}}(\cdot|x)\big)\Big]\right\} \\
            &\qquad = \eta\cdot \max_{\pi\in\Pi}\left\{C_r - \beta\cdot \mathbb{E}_{x\sim d_0}\left[D_{\mathrm{KL}}\left(\pi(\cdot|x)\middle\| \frac{\pi^{\mathrm{ref}}(\cdot|x)\cdot \exp(\beta^{-1}\cdot r(x,\cdot))}{\int_{a'\in\cA}\mathrm{d}\pi^{\mathrm{ref}}(a'|x)\cdot \exp(\beta^{-1}\cdot r(x,a'))}\right)\right]\right\},\label{eq: proof unique 1}
        \end{align}
        where 
        $$C_r=\mathbb{E}_{x\sim d_0}\left[\beta\cdot\log\left(\int_{a\in\cA}\mathrm{d}\pi^{\mathrm{ref}}(a|x)\cdot \exp\left(\beta^{-1}\cdot r(x,a)\right)\right)\right]$$ 
        is a constant independent of $\pi$. 
        Therefore, the maximizer of $\phi(\cdot,r)$ on the support of $d_0$ must equal to 
        \begin{align}
            \pi_r(\cdot|x) = \frac{\pi^{\mathrm{ref}}(\cdot|x)\cdot \exp(\beta^{-1}\cdot r(x,\cdot))}{\int_{a'\in\cA}\mathrm{d}\pi^{\mathrm{ref}}(a'|x)\cdot \exp(\beta^{-1}\cdot r(x,a'))},
        \end{align}
        which completes the proof of Lemma~\ref{lem: unique pi}.
    \end{proof}
\end{lemma}

\section{Proofs for Generalization to New Prompt Distributions}\label{sec: proof generalization}

\begin{proof}[Proof of Corollary~\ref{cor: context generalization}]
    Consider by Holder's inequality, we can bound the new target as 
    \begin{align}
        &\mathbb{E}_{x\sim d_1, a\sim \pi^{\star}(\cdot|x)}[r^{\star}(x,a)] -  \mathbb{E}_{x\sim d_1, a\sim \pi_{\widehat{r}}(\cdot|x)}[r^{\star}(x,a)] \\
        &\qquad= \mathbb{E}_{x\sim d_1}\left[\mathbb{E}_{a\sim \pi^{\star}(\cdot|x)}[r^{\star}(x,a)] -  \mathbb{E}_{a\sim \pi_{\widehat{r}}(\cdot|x)}[r^{\star}(x,a)]\right] \\
        &\qquad = \mathbb{E}_{x\sim d_0}\left[\frac{d_1(x)}{d_0(x)}\cdot \left(\mathbb{E}_{a\sim \pi^{\star}(\cdot|x)}[r^{\star}(x,a)] -  \mathbb{E}_{a\sim \pi_{\widehat{r}}(\cdot|x)}[r^{\star}(x,a)]\right)\right] \\ 
        &\qquad \leq C_{\infty}(d_0, d_1)\cdot \mathbb{E}_{x\sim d_0}\left[\left|\mathbb{E}_{a\sim \pi^{\star}(\cdot|x)}[r^{\star}(x,a)] -  \mathbb{E}_{a\sim \pi_{\widehat{r}}(\cdot|x)}[r^{\star}(x,a)]\right|\right],\label{eq: proof generalization 1}
    \end{align}
    where the last inequality follows from Holder's inequality and the definition of the density ratio $C_{\infty}(d_0, d_1)$ in Corollary~\ref{cor: context generalization}. 
    Now since $\pi^{\star}$ is the optimal policy in $\Pi$ in the sense that $\pi^{\star} = \argmax_{\pi\in\Pi}\mathbb{E}_{a\sim \pi(\cdot|x)}[r^{\star}(x,a)]$ for any $x\in\cX$, we have that 
    \begin{align}
        &\mathbb{E}_{x\sim d_0}\left[\left|\mathbb{E}_{a\sim \pi^{\star}(\cdot|x)}[r^{\star}(x,a)] -  \mathbb{E}_{a\sim \pi_{\widehat{r}}(\cdot|x)}[r^{\star}(x,a)]\right|\right] \\
        &\qquad =\mathbb{E}_{x\sim d_0, a\sim \pi^{\star}(\cdot|x)}[r^{\star}(x,a)] -  \mathbb{E}_{x\sim d_0, a\sim \pi_{\widehat{r}}(\cdot|x)}[r^{\star}(x,a)] \\
        &\qquad = \mathrm{Gap}^{\pi^{\star}}(\pi_{\widehat{r}}),\label{eq: proof generalization 2}
    \end{align}
    Therefore, by combining \eqref{eq: proof generalization 1} and \eqref{eq: proof generalization 2}, invoking Corollary~\ref{cor: subopt rdpo}, we have that
    \begin{align}
        \mathbb{E}_{x\sim d_1, a\sim \pi^{\star}(\cdot|x)}[r^{\star}(x,a)] -  \mathbb{E}_{x\sim d_1, a\sim \pi_{\widehat{r}}(\cdot|x)}[r^{\star}(x,a)] \leq\widetilde{\cO}\left(\frac{C_{\infty}(d_0, d_1)}{\sqrt{N}}\right),
    \end{align}
    with probability at least $1-\delta$. This proves Corollary~\ref{cor: context generalization}.
\end{proof}

\section{Additional Details on Experiments}
\subsection{Training Details}
\label{app:training_details}
We train the gemma series models with 8 NVIDIA A6000 GPUs and the beta series models with 8 NVIDIA A100 GPUs, where they are all GPT-like models with around 7 billion parameters.  It takes around three hours to train a beta series model and five hours to train a gemma one. 
Our codebase is adapted from the Alignment Handbook \citep{alignment_handbook2023}. By comparing the validation loss on the test split (not used for later evaluation), we select the hyperparameter $\eta$ of both RPO (beta) and RPO (gemma) to be $0.005$. We list the remaining training configurations in Table \ref{tab:hype}, which are recommended by the Alignment Handbook. 
\begin{table}[H]
    \centering
    \begin{tabular}{c |cc}
    \toprule 
        Configuration & Beta Series & Gemma Series\\
        \midrule
        learning rate & 5.0e-7 & 5.0e-7\\
        learning scheduler type& cosine &cosine \\
        warmup ratio & 1.0 &1.0\\
        batch size &128 &128\\
        gradient accumulation & 2 & 16\\
        batch size per device & 8 & 1 \\
        training epoch & 1 &2
        \\
        $\beta$ & 0.01 &0.05\\
        optimizer & adamw torch & adamw torch \\
        seed& 42 &42 \\
        precision & bfloat16 & bfloat16
        \\
         \bottomrule
    \end{tabular}
    \vspace{1mm}
    \caption{Training configurations for beta series and gemma series models in this paper.} 
    \label{tab:hype}
\end{table}

\subsection{Evaluation Details}
\label{app:eval}

\paragraph{GPT-4 evaluation on the test split.} We use the following prompts to guide GPT-4 (0613) to annotate the preferences among win, lose, and tie (we denote them by A, B, and C, respectively). 
\begin{tcolorbox}[ colback=pink!10, boxrule=0.15mm]
    \textbf{Prompts:} Please act as an impartial judge and evaluate the quality of the responses provided by two AI assistants to the user question displayed below. You should choose the assistant that follows the user's instructions and answers the user's question better. Your evaluation should consider factors such as the helpfulness, relevance, accuracy, depth, creativity, and level of detail of their responses. Begin your evaluation by comparing the two responses and provide a short explanation. Avoid any position biases and ensure that the order in which the responses were presented does not influence your decision. Do not allow the length of the responses to influence your evaluation. Do not favor certain names of the assistants. Be as objective as possible. After providing your explanation, output your final verdict by strictly following this format: [[A]] if assistant A is better, [[B]] if assistant B is better, and [[C]] for a tie.
[Instruction]
{instruction}
[The Start of Assistant A's Answer]
\{\textit{answer A}\}
[The End of Assistant A's Answer]
[The Start of Assistant B's Answer]
\{\textit{answer B}\}
[The End of Assistant B's Answer]
\end{tcolorbox}
Here, we replace \{\textit{answer A}\} and \{\textit{answer B}\} with the answers of two models. Since GPT annotation has shown to prefer the answer in the first position \citep{wang2023large}, we randomly exchange the positions between two answers during the evaluation to ensure a fair comparison. 

\paragraph{Benchmark evaluation.} We use the default configuration for the evaluations on MT-Bench\footnote{\url{https://github.com/lm-sys/FastChat/tree/main/fastchat/llm_judge}}  and AlpacaEval 2.0\footnote{\url{https://github.com/tatsu-lab/alpaca_eval/tree/main}}. By default, the annotator of MT-Bench is the \textit{latest version} of GPT-4. The default annotator and the competitor model are both GPT-4 (Preview 11/06).   
We only need to manually import the proper chat template that formats the training dataset, which are shown as follows. 
\begin{tcolorbox}[colback=pink!10, boxrule=0.15mm]\textbf{Chat Template for Beta Series:}
\textless\textbar system\textbar\textgreater\textless/s\textgreater\textless\textbar user\textbar\textgreater\\
\{\textit{instruction}\}\textless/s\textgreater\\
\textless\textbar assistant\textbar\textgreater\\
\end{tcolorbox}
\begin{tcolorbox}[colback=pink!10, boxrule=0.15mm]\textbf{Chat Template for Gemma Series:}
\textless bos\textgreater \textless\textbar im\_start\textbar\textgreater user\\
\{\textit{instruction}\}\textless\textbar im\_end\textbar\textgreater\\
\textless\textbar im\_start\textbar\textgreater assistant\\
\end{tcolorbox}

\subsection{Additional Results on Experiments}\label{app:more_results}
In this section, we provide the additional results to show the performance gain for RPO (beta) in MT-Bench and RPO (gemma) in AlpacaEval 2.0.  We report the pairwise win rates in Tables \ref{tab:mt_pair_beta}, \ref{tab:alpaca_pair_gemma}, and \ref{tab:alpaca_pair_gemma2} to analyze their performance gaps, where all the annotation configurations are the same in Table \ref{tab:benchmark}. Results show that RPO still exceeds DPO in the metric of the pairwise win rates on the benchmarks for both beta series and gemma series. 
\label{app:exp}

\begin{table}[H]
    \centering
    {%
    \begin{tabular}{c | c | c |c  }
    \toprule
       win rate (\%) & RPO (beta) & Ref.  (beta) & DPO (beta)  \\
        \midrule
         RPO (beta) & 50.00 & \textcolor{blue}{\textbf{83.75}}  &  \textcolor{blue}{\textbf{57.81}} \\
        Ref. (beta)& 16.25 & 50.00 & 21.25 \\
        DPO (beta)& 78.75 & 42.19 & 50.00 \\
    \bottomrule
    \end{tabular}%
    }
    \vspace{1mm}
\caption{ Pairwise win rates (left vs. right) for beta series models on MT-Benchmark.}
   \label{tab:mt_pair_beta}
\end{table}
\begin{table}[H]
    \centering
    {%
    \begin{tabular}{c | c | c |c  }
    \toprule
        win rate (\%) & RPO (beta) & Ref. (beta) & DPO (beta)  \\
        \midrule
         RPO(beta) & 50.00 & \textcolor{blue}{\textbf{80.13}}  & \textcolor{blue}{\textbf{52.02}} \\
         Ref.(beta)& 19.87 & 50.00 & 20.61 \\
        DPO (beta)& 47.98 & 79.39 & 50.00\\
    \bottomrule
    \end{tabular}%
    }
    \vspace{1mm}
\caption{ Pairwise win rates (left vs. right)  for gemma series models on AlpacaEval 2.0.}
   \label{tab:alpaca_pair_gemma}
\end{table}
\begin{table}[H]
    \centering
    {%
    \begin{tabular}{c | c | c |c  }
    \toprule
        win rate (\%) & RPO (beta) & Ref. (beta) & DPO (beta)  \\
        \midrule
         RPO (beta) & 50.00 & \textcolor{blue}{\textbf{64.93}}  &  \textcolor{blue}{\textbf{51.33}} \\
        Ref. (beta)& 35.07 & 50.00 & 36.44 \\
        DPO (beta)& 48.67 & 64.56 & 50.00 \\
    \bottomrule
    \end{tabular}%
    }
    \vspace{1mm}
\caption{ Pairwise Length-Control (LC) win rates (left vs. right) for gemma series models on AlpacaEval 2.0.}
   \label{tab:alpaca_pair_gemma2}
\end{table}

\section{Experiments on  Math, Reasoning, and Coding Tasks}\label{subsec: further experiments}

\subsection{Experimental Details} To provide a more comprehensive analysis of the trained LLM, we introduce more benchmarks on the math, reasoning, and coding tasks for evaluations. 
Specifically, we choose the Grade School Math 8K (GSM8K), AI2 Reasoning Challenge (ARC), and Mostly Basic Python Programming (MBPP) to measure math, reasoning, and coding abilities, respectively. 
In this section, we focus on the gemma series for the experiments.
We do not use chain-of-thought or few shots in all the benchmarks. 
We compare the greedy decoding result (pass $@1$) on the MBPP  benchmark.

Here we use the  OpenRLHF codebase \citep{hu2024openrlhf} to implement a new variant of RPO, where the SFT loss regularizer is averaged by the number of tokens of the chosen labels, that is, $(\log \pi_\theta(a_{\text{cho}}|x))/|a_{\text{cho}}|$. Such a variant balances the weight of the averaged SFT loss regularizer between the shorter chosen response and the longer one. 
We set the coefficient for the SFT loss regularizer as $0.2$. 
We use 8 NVIDIA A100 GPUs for the training and evaluation. 
The remaining hyperparameters are in Table~\ref{tab:hype2}.
\begin{table}[H]
    \centering
    \begin{tabular}{c |c}
    \toprule 
        Configuration & Gemma Series\\
        \midrule
        learning rate & 5.0e-7 \\
        learning scheduler type& {cosine with a minimum learning rate} \\
        batch size &128\\
        gradient accumulation & 8\\
        batch size per device & 2 \\
        training epoch  &2
        \\
        $\beta$ & 0.5\\
        optimizer & adamw torch  \\
        seed& 42 \\
        precision & bfloat16 
        \\
         \bottomrule
    \end{tabular}
    \vspace{1mm}
    \caption{Training configurations for DPO and RPO for the experiments in Appendix \ref{subsec: further experiments}.} 
    \label{tab:hype2}
\end{table}

\subsection{Experimental Results}
Table \ref{tab:benchmark2} demonstrates that our proposed method still outperforms or performs equally to the vanilla DPO on these benchmarks of math, reasoning, and coding, which verifies the effectiveness of our proposed method.

\begin{table}[h]
    \centering

    {
    \begin{tabular}{    ccc ccc}
    \toprule
    \multirow{2}{*}{Model Name} &GSM8K  & \multicolumn{2}{c}{ARC}  & \multicolumn{2}{c}{MBPP (Pass $@1$)} \\
    \cline{3-4}\cline{5-6}
        & (\%) & Easy (\%) & Challenge (\%)&
         Normal (\%) & Plus (\%)\\
        \midrule
        RPO & \textcolor{blue}{\textbf{49.9}} & 
\textcolor{blue}{\textbf{79.1}} &
{49.8} &
54.2 &\textcolor{blue}{\textbf{46.3}}\\
        DPO & 45.3 &
 75.7 &
\textcolor{blue}{\textbf{50.0}} & 
54.2 & 43.9 \\
  Ref.   & 45.4 & 
75.0 & 45.8 
& 50.3 & 44.2  \\
         \texttt{zephyr-gemma-7b}& 47.3 
&77.6 &
48.6 
&
\textcolor{blue}{\textbf{54.5}} &44.7\\
    \bottomrule
    \end{tabular}%
    }
    \vspace{1mm}
    \caption{The results on GSM8K, ARC, and MBPP. Here, \texttt{zephyr-gemma-7b} is the officially released models trained by DPO and Ref. denotes the reference model \texttt{zephyr-7b-gemma-sft} used for our training. 
    RPO and DPO are trained with the  OpenRLHF codebase \citep{hu2024openrlhf} and we average the SFT loss regularizer in RPO by the number of tokens of the chosen response. We do not use chain-of-thought or few shots in all the benchmarks. 
    We compare the greedy decoding result (pass $@1$) for MBPP.}
    \label{tab:benchmark2}
\end{table}

\end{document}